\documentclass[10pt,twocolumn,letterpaper]{article}

\usepackage{cvpr}              

\usepackage[table,dvipsnames]{xcolor}
\usepackage{makecell}
\usepackage{tcolorbox}
\usepackage{enumerate}

\usepackage[shortlabels]{enumitem}
\usepackage{amsmath}
\usepackage{amssymb}
\usepackage{booktabs}
\usepackage{multirow}
\usepackage{times}
\usepackage{lipsum}
\usepackage{subcaption}
\usepackage{diagbox}

\usepackage{bm,cuted}
\usepackage[ruled,vlined,linesnumbered]{algorithm2e}
\usepackage{graphicx}
\usepackage{tabularx} 
\usepackage[safe]{tipa}
\usepackage{multirow}
\usepackage{xcolor}
\usepackage{dsfont}
\usepackage{upgreek}
\usepackage{url}

\usepackage{pifont}
\makeatletter
\newcommand\HUGE{\@setfontsize\Huge{20}{30}}
\makeatother

\def\eqref#1{equation~\ref{#1}}

\def\1{\bm{1}}

\DeclareMathAlphabet{\mathsfit}{\encodingdefault}{\sfdefault}{m}{sl}
\SetMathAlphabet{\mathsfit}{bold}{\encodingdefault}{\sfdefault}{bx}{n}

\usepackage[accsupp]{axessibility}
\usepackage{amsthm}

\newtheorem{proposition}{Proposition}

\theoremstyle{remark}
\newtheorem{remark}{Remark}[proposition]

\providecommand{\FullStop}{\text{~\@.\xspace}}
\providecommand{\Comma}{\text{~,\xspace}}

\usepackage{xspace}
\usepackage{hhline}   

\usepackage{graphicx}
\graphicspath{{./images/}}

\providecommand{\kuka}{\textsc{KUKA} LBR iiwa R820\xspace}
    
\usepackage[pagebackref,breaklinks,colorlinks,bookmarks=false]{hyperref}

\hypersetup{pdfstartview={FitH}}

\usepackage[capitalize]{cleveref}
\crefname{section}{Sec.}{Secs.}
\Crefname{section}{Section}{Sections}
\Crefname{table}{Table}{Tables}
\crefname{table}{Tab.}{Tabs.}

\def\cvprPaperID{4409} 
\def\confName{CVPR}
\def\confYear{2024}

\begin{document}


\title{Language-driven Grasp Detection}
\author{An Dinh Vuong$^{1}$, Minh Nhat Vu$^{2,*}$, Baoru Huang$^{3,*}$, Nghia Nguyen$^{1}$, Hieu Le$^{1}$, Thieu Vo$^{4}$, Anh Nguyen$^{5}$\\
{\small $^{1}$FPT Software AI Center, Vietnam}
{\small $^{2}$Automation \& Control Institute, TU Wien, Austria}
{\small $^{3}$Imperial College London, UK} \\
{\small $^{4}$Ton Duc Thang University, Vietnam}
{\small $^{5}$University of Liverpool, UK}
{\small $^{*}$Co-Corresponding authors}\\
{\small \href{https://airvlab.github.io/grasp-anything}{https://airvlab.github.io/grasp-anything}}}


\twocolumn[{%
\renewcommand\twocolumn[1][]{#1}%
\maketitle
\begin{center}
  \centering
  \vspace{-3ex}
  \captionsetup{type=figure}
  \Large
\resizebox{\linewidth}{!}{
\setlength{\tabcolsep}{2pt}
\begin{tabular}{ccccc}
\shortstack{\includegraphics[width=0.33\linewidth]{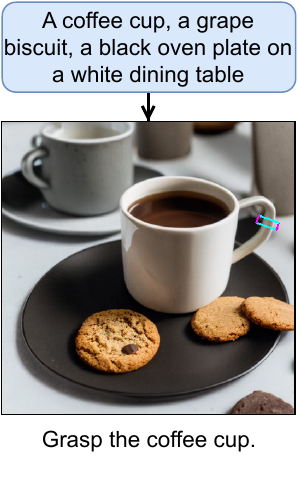}}&
\shortstack{\includegraphics[width=0.33\linewidth]{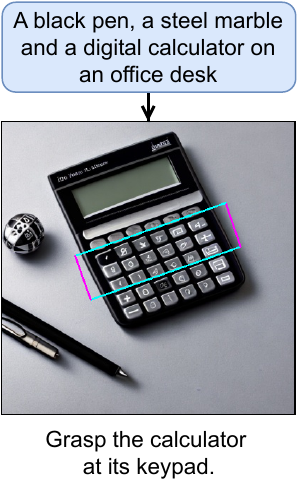}}&
\shortstack{\includegraphics[width=0.33\linewidth]{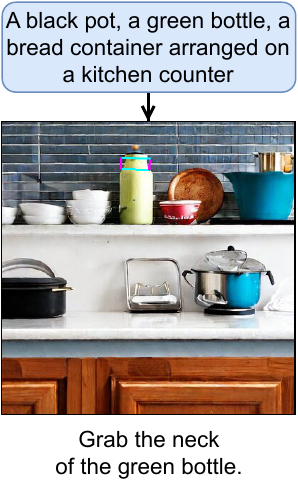}}&
\shortstack{\includegraphics[width=0.33\linewidth]{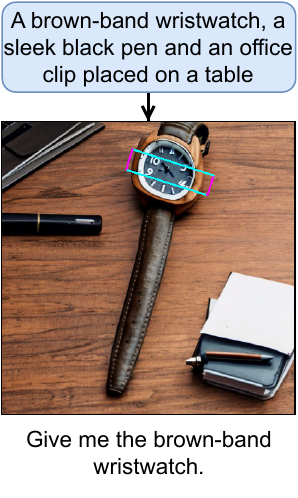}}&
\shortstack{\includegraphics[width=0.33\linewidth]{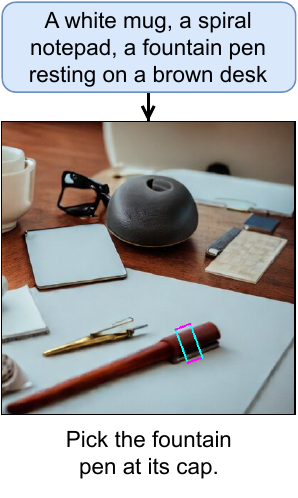}}\\[3pt]
\end{tabular}
}
\vspace{-2ex}
    \captionof{figure}{We present a new dataset and method for \textit{language-driven grasp} task.}
    \label{fig:IntroVis}
\end{center}%
}]


\begin{abstract}
\vspace{-3ex}
Grasp detection is a persistent and intricate challenge with various industrial applications. Recently, many methods and datasets have been proposed to tackle the grasp detection problem. However, most of them do not consider using natural language as a condition to detect the grasp poses. In this paper, we introduce Grasp-Anything++, a new language-driven grasp detection dataset featuring 1M samples, over 3M objects, and upwards of 10M grasping instructions. We utilize foundation models to create a large-scale scene corpus with corresponding images and grasp prompts. We approach the language-driven grasp detection task as a conditional generation problem. Drawing on the success of diffusion models in generative tasks and given that language plays a vital role in this task, we propose a new language-driven grasp detection method based on diffusion models. Our key contribution is the contrastive training objective, which explicitly contributes to the denoising process to detect the grasp pose given the language instructions. We illustrate that our approach is theoretically supportive. The intensive experiments show that our method outperforms state-of-the-art approaches and allows real-world robotic grasping. Finally, we demonstrate our large-scale dataset enables zero-short grasp detection and is a challenging benchmark for future work. 

%

\end{abstract}

\section{Introduction}
Imagine we want an assistant robot to grasp a cup among a clutter of daily objects such as a knife, a fork, a cup, and a pair of scissors. Conventionally, to convey the idea of grasping this specific object, humans use the natural language command, ``give me the cup", for instance. Although humans intuitively know how to grasp the cup given the linguistic command, determining specific grasp actions for objects based on natural language instructions or \textit{language-driven grasp detection} remains challenging for robots~\cite{shridhar2022cliport}. First, natural language is usually overlooked in existing grasp datasets~\cite{platt2023grasp} while training vision-and-language neural networks necessitates an excessive number of labeled examples~\cite{song2023llm}. Second, recent works usually focus on particular manipulation tasks with limited objects~\cite{gilles2022metagraspnet}, imposing a bottleneck for in-the-wild robot execution~\cite{pmlr-v205-shah23b}. Finally, despite recent developments, bridging the gap between language, vision, and control for real-world robotic experiments remains a challenging task~\cite{yang2023pave}.


Recently, language-driven robotic frameworks are gaining traction, offering the potential for robots to process natural language, and bridging the gap between robotic manipulations and real-world human-robot interaction~\cite{mu2023embodiedgpt}. PaLM-E~\cite{driess2023palm}, EgoCOT~\cite{mu2023embodiedgpt}, and ConceptFusion~\cite{jatavallabhula2023conceptfusion} are some notable embodied robots with the ability to comprehend natural language by harnessing the power of large foundation models such as ChatGPT~\cite{openai2021chatgpt}. However, most works assume the high-level actions of robots and ignore the fundamental grasping actions, restricting the structure for generalization across robotic domains, tasks, and
skills~\cite{mu2023ec2}. In this paper, we explore training a language-driven agent to implement low-level actions, focusing on the task of object grasping via image observations. Specifically, our hypothesis is centered around the establishment of a robotic system that can execute grasping actions following a given language instruction for any universal object.

We first present \textbf{Grasp-Anything++} to serve as a large-scale dataset for language-driven grasp detection. Our dataset is based on the Grasp-Anything~\cite{vuong2023grasp}, and is synthesized from foundation models. Compared to the original Grasp-Anything dataset, we provide more than 10M grasp prompts and 3M associated object masks, 6M ground truth poses at the object part level. Our dataset showcases the ability to facilitate grasp detection using language instructions. We label the ground truth at both the object level and part level, providing a comprehensive understanding of real-world scenarios. For example, our ground truth includes both general instructions ``\textit{give me the knife}" and detail ones such as ``\textit{grasp the handle of the steak knife}". We empirically show that our large-scale dataset successfully facilitates zero-shot grasp detection on both vision-based tasks and real-world robotic experiments.


To tackle the challenging language-driven grasp detection task, we propose a new diffusion model-based method. Our selection of diffusion models is motivated by their proven efficacy in conditional generation tasks~\cite{ho2020denoising}. These models have shown efficiency beyond image synthesis, including other image-based tasks such as 
image segmentation~\cite{wolleb2022diffusion}, and visual grounding~\cite{li2023gligen}. Despite achieving notable success, integrating visual and text features effectively remains a challenge~\cite{chen2023language} as the majority of existing literature employs latent strategies to combine visual and text features~\cite{lee2023text}. We address this challenge by employing a new training strategy for learning text and image features, focusing on the use of feature maps as guidance information for grasp pose generation. Our main contribution is a new training objective that incorporates the feature maps and explicitly contributes to the denoising process. In summary, our contributions are three-fold:

\begin{itemize}[leftmargin=*]
    \item We propose Grasp-Anything++, a large-scale language-driven dataset for grasp detection tasks.
    \item We propose a diffusion model with a training objective that explicitly contributes to the denoising process to detect the grasp poses.
    \item We demonstrate that our Grasp-Anything++ dataset and the proposed method outperform other approaches and enable successful robotic applications. 
\end{itemize}
\section{Related Work}
\textbf{Grasp Detection.} Grasp detection is a popular task in both computer vision and robotic community~\cite{depierre2018jacquard,nguyen2016preparatory,shridhar2022cliport, ainetter2021end, fang2020graspnet}. Recently, establishing robotic systems with the ability to follow natural commands has been actively researched~\cite{yang2023pave, shridhar2022cliport, xu2023joint}. The prevalent solution to the language-driven grasp detection task is to split into two stages: one for grounding the target object, and the other is to synthesize grasp poses from the grounding visual-text correlations~\cite{xu2023joint, ahn2022can}. Training in two stages may result in longer inference time~\cite{liu2023target}. In addition, several works~\cite{yang2023pave,xu2023object, zheng2022vlmbench} adopt foundation models, such as GroundDINO~\cite{liu2023grounding} and GPT-3~\cite{brown2020language}. Accessing such commercial foundation models is not always available~\cite{vemprala2023chatgpt}, especially on robotic systems with limited resources or unstable internet connection~\cite{lee2022task}. 
In our work, we directly train the model on the large-scale Grasp-Anything++ dataset to inherit the power of a foundation-based dataset, while ensuring a straightforward inference process for the downstream robotic applications.

\textbf{Language-driven Grasp Detection Datasets.} While there are many grasp datasets have been introduced~\cite{jiang2011efficient, depierre2018jacquard, pinto2016supersizing, levine2018learning, xiang2017posecnn, mahler2017dex, mousavian20196, eppner2019billion, fang2020graspnet, morrison2020egad, eppner2021acronym, cao2023nbmod}, the majority of them overlook the text modality. Therefore, the grasping of objects out of a clutter typically experiences ambiguities in what object to grasp~\cite{zhang2021invigorate}. 
DailyGrasp~\cite{yang2023pave} is one of the first grasp datasets employing natural language for scene descriptions; however, the scene description corpus in this dataset is relatively small and does not specify which part of the object should be grasped. In our work, we present Grasp-Anything++, which is a large-scale language-driven grasping dataset. 
Furthermore, Grasp-Anything++ describes the grasping object at both the part level and object level, providing more information for the robot to execute the grasping~\cite{yang2022interactive}.

\begin{table*}[!hbt]
\footnotesize
\centering
\setlength\tabcolsep{2pt}
\setlength\extrarowheight{1pt}
\begin{tabular}{p{0.08\linewidth} p{0.065\linewidth} p{0.68\linewidth} p{0.14\linewidth}} %
\toprule 
\textbf{Step} & \multicolumn{3}{c}{\textbf{Description} }
\\
\midrule

\multirow{3}{*}{\textbf{\makecell[l]{Scene\\ Generation}} } & 
\textbf{User} & 
Please help me generate scene descriptions for natural arrangements of daily objects. Each description has the following form: {\textless{\texttt{Object\_1}}\textgreater}{\textless{\texttt{Object\_2}}\textgreater}...{\textless{\texttt{Verb}}\textgreater}{\textless{\texttt{Container\_Object}}\textgreater}. Please also ensure the incorporation of a rich and varied lexicon in the scene descriptions. & \multirow{3}{*}{\textbf{\makecell[c]{\includegraphics[width=2.43cm]{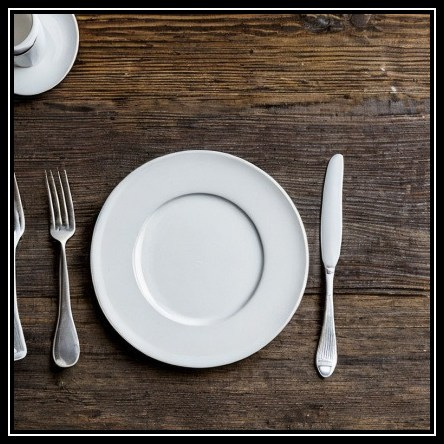}}} }
\\ \cmidrule{2-3}
& \textbf{Sample} & A steel knife, a polished fork and a pristine ceramic plate on a wooden table. & 
\\ \cmidrule{2-3}
& \textbf{\makecell[l]{Text-to-\\Image}} & We use Stable Diffusion~\cite{rombach2022high} to proceed text-to-image generation. & 
\\ 
\midrule
\multirow{3}{*}{\textbf{\makecell[l]{Object\\ Masking}} } & 
\textbf{User} & 
For object part-level description, given an input list 
\{{\textless{\texttt{Object\_1}}\textgreater}, {\textless{\texttt{Object\_2}}\textgreater}, ...\}, the output will be a list that describes the parts of objects as:
\{{{\textless{\texttt{Object\_1}}\textgreater}: [\textless\textcolor{RoyalBlue}{\texttt{Part\_1.1}}\textgreater, \textless\textcolor{RoyalBlue}{\texttt{Part\_1.2}}\textgreater, ...], {\textless{\texttt{Object\_2}}\textgreater}: [\textless\textcolor{RoyalBlue}{\texttt{Part\_2.1}}\textgreater, \textless\textcolor{RoyalBlue}{\texttt{Part\_2.2}}\textgreater, ...]}\}. & \multirow{3}{*}{\textbf{\makecell[c]{\includegraphics[width=2.43cm]{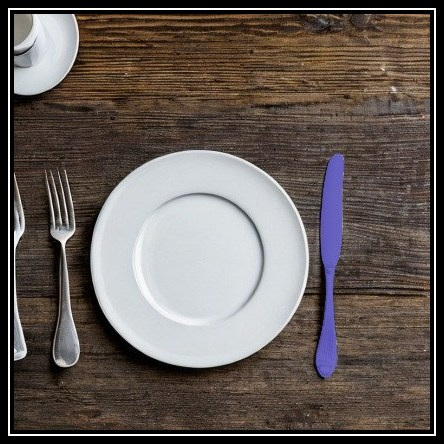}}} }
\\ \cmidrule{2-3}
&
\textbf{Sample} & \{{{{\texttt{knife}}}: [\textcolor{RoyalBlue}{\texttt{handle}}, \textcolor{RoyalBlue}{\texttt{blade}}], {{\texttt{fork}}}: [\textcolor{RoyalBlue}{\texttt{handle}}, \textcolor{RoyalBlue}{\texttt{neck}}, \textcolor{RoyalBlue}{\texttt{stem}}, \textcolor{RoyalBlue}{\texttt{tines}}]}, {{\texttt{plate}}}: [\textcolor{RoyalBlue}{\texttt{rim}}, \textcolor{RoyalBlue}{\texttt{base}}]\}& 
\\ \cmidrule{2-3}
&
\textbf{\makecell[l]{Post\\ Process}} & We use OFA~\cite{wang2022ofa} and SAM~\cite{kirillov2023segment} to locate the region describing the objects.
\\
\midrule
\multirow{3}{*}{\textbf{\makecell[l]{Part\\ Masking}} } & 
\textbf{User} & Given the object list and part lists of each scene description, you will generate for me all prompts with the following format: \{\textless{\texttt{Manipulation\_Action}}\textgreater\textless{\texttt{Object\_ID}}\textgreater\textless\textcolor{RoyalBlue}{\texttt{Part\_ID}}\textgreater\}. The part that is more suitable for human grasping is positioned at the start of the list to represent the grasping actions. & \multirow{3}{*}{\textbf{\makecell[c]{\includegraphics[width=2.43cm]{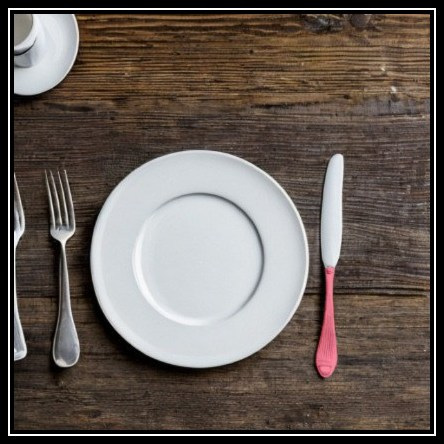}}} } 
\\ \cmidrule{2-3}
&
\textbf{Sample} & Give me the steel knife; Grasp the knife at its handle. & 
\\ \cmidrule{2-3}
&
\textbf{\makecell[l]{Post\\ Process}} & We leverage VLPart~\cite{vlpart2023} to locate the region describing the parts of objects.
\\
\midrule
\multirow{3}{*}{\textbf{\makecell[l]{Grasp\\ Generation}} } & 
\textbf{User} & Generate for me a scene description with grasp instructions following the templates. & \multirow{3}{*}{\textbf{\makecell[c]{\includegraphics[width=2.43cm]{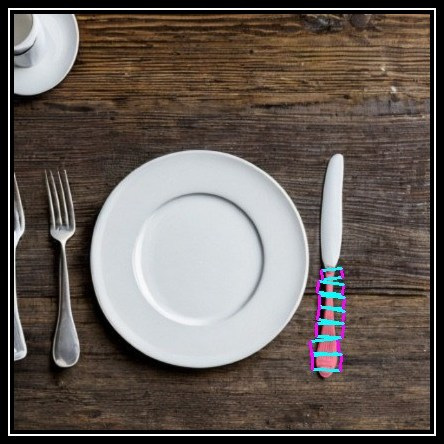}}} }
\\ \cmidrule{2-3}
&
\textbf{Sample} & 
\textbf{Scene description}: A steel knife, a polished fork and a pristine ceramic plate on a wooden table. \textbf{Object list:} \{\texttt{knife}, \texttt{fork}, \texttt{plate}\}. \textbf{Part lists:} \{{{{\texttt{knife}}}: [\textcolor{RoyalBlue}{\texttt{handle}}, \textcolor{RoyalBlue}{\texttt{blade}}], {{\texttt{fork}}}: [\textcolor{RoyalBlue}{\texttt{handle}}, \textcolor{RoyalBlue}{\texttt{neck}}, \textcolor{RoyalBlue}{\texttt{stem}}, \textcolor{RoyalBlue}{\texttt{tines}}]}, {{\texttt{plate}}}: [\textcolor{RoyalBlue}{\texttt{rim}}, \textcolor{RoyalBlue}{\texttt{base}}]\}. \textbf{Prompts:} Give me the steel knife; Grasp the knife at its handle. & 
\\ \cmidrule{2-3}
&
\textbf{\makecell[l]{Grasp\\ Labelling}} & We utilize a pretrained RAGT-3/3~\cite{cao2023nbmod} to generate grasp poses corresponding to the located region.
\\
\bottomrule
\end{tabular}
\vspace{-2ex}
\caption{\textbf{Grasp-Anything++ creation pipeline.} We utilize ChatGPT to generate scene descriptions and grasp instructions from the user input. We generate images given scene descriptions and automatically synthesize the grasp poses.
}
\label{table:prompt-example}
\vspace{-4ex}
\end{table*}

\textbf{Diffusion Models for Robotic Applications.} Diffusion models~\cite{ho2020denoising} have emerged as the new state-of-the-art method of generative tasks~\cite{yang2022diffusion}. Recently, we have witnessed growing attention for utilizing diffusion models in robotic applications~\cite{saha2023edmp}. Liu \textit{et al.}~\cite{liu2022structdiffusion} propose a diffusion model to handle the language-guided object rearrangement task. Diffusion models are also applied to other robotic tasks such as motion planning~\cite{carvalho2023motion}, and trajectory optimization~\cite{janner2022planning}. The authors in~\cite{urain2023se} present a diffusion model to determine grasp poses by minimizing a SDF loss. Overall, the diffusion models employed in previous works often combine visual and text features in a latent mechanism~\cite{blattmann2023align}, which may cause interpretability problems~\cite{yang2023pave} for robotic systems that require low-level controls~\cite{doshi2017towards}. To tackle this challenge, we propose a training objective that explicitly contributes to the denoising process. We demonstrate that our proposed strategy is theoretically supported and is more effective than the latent strategy.

\section{The Grasp-Anything++ Dataset}
We utilize large-scale foundation models to create the Grasp-Anything++. Our dataset offers open-vocabulary grasping commands and images with associated groundtruth. 
There are three key steps in establishing our dataset: \textit{i)} prompting procedure, \textit{ii)} image synthesis and grasp poses annotation, and \textit{iii)} post-processing.



\subsection{Prompting Procedure}
We first establish prompt-based procedures to generate a large-scale scene description corpus as well as grasp prompt instructions. In particular, we utilize ChatGPT to generate the prompts for two tasks: \textit{i)} Scene descriptions: Sentences capturing the scene arrangement, including the extracted object and part lists, and \textit{ii)} Grasp instructions: Prompts directing the robot to grasp specific objects or parts. 


We follow a procedure in Table~\ref{table:prompt-example} to implement ChatGPT's output templates. 
The reference target in the grasp instruction may be either an object or an object's part. 
When the reference is an object's part, that part is directly selected as the reference in the grasp instruction sentence. If the reference is an object, we determine the grasping region on the part of the object that is likely to be grasped in everyday scenarios as described in affordance theory~\cite{nguyen2023open}. 


\subsection{Image Synthesis and Grasp Annotation}
\textbf{Image Synthesis.} Given the scene description corpus, we first utilize a large-scale pretrained text-to-image model, namely, Stable Diffusion~\cite{rombach2022high} to generate images from scene descriptions. Next, we perform a series of visual grounding and image segmentation using OFA~\cite{wang2022ofa}, Segment-Anything~\cite{kirillov2023segment}, and VLPart~\cite{vlpart2023} to locate the referenced object or part to the grasp instruction.

\begin{figure*}[ht]
  \begin{subfigure}{0.31\linewidth}
    \includegraphics[width=\linewidth]{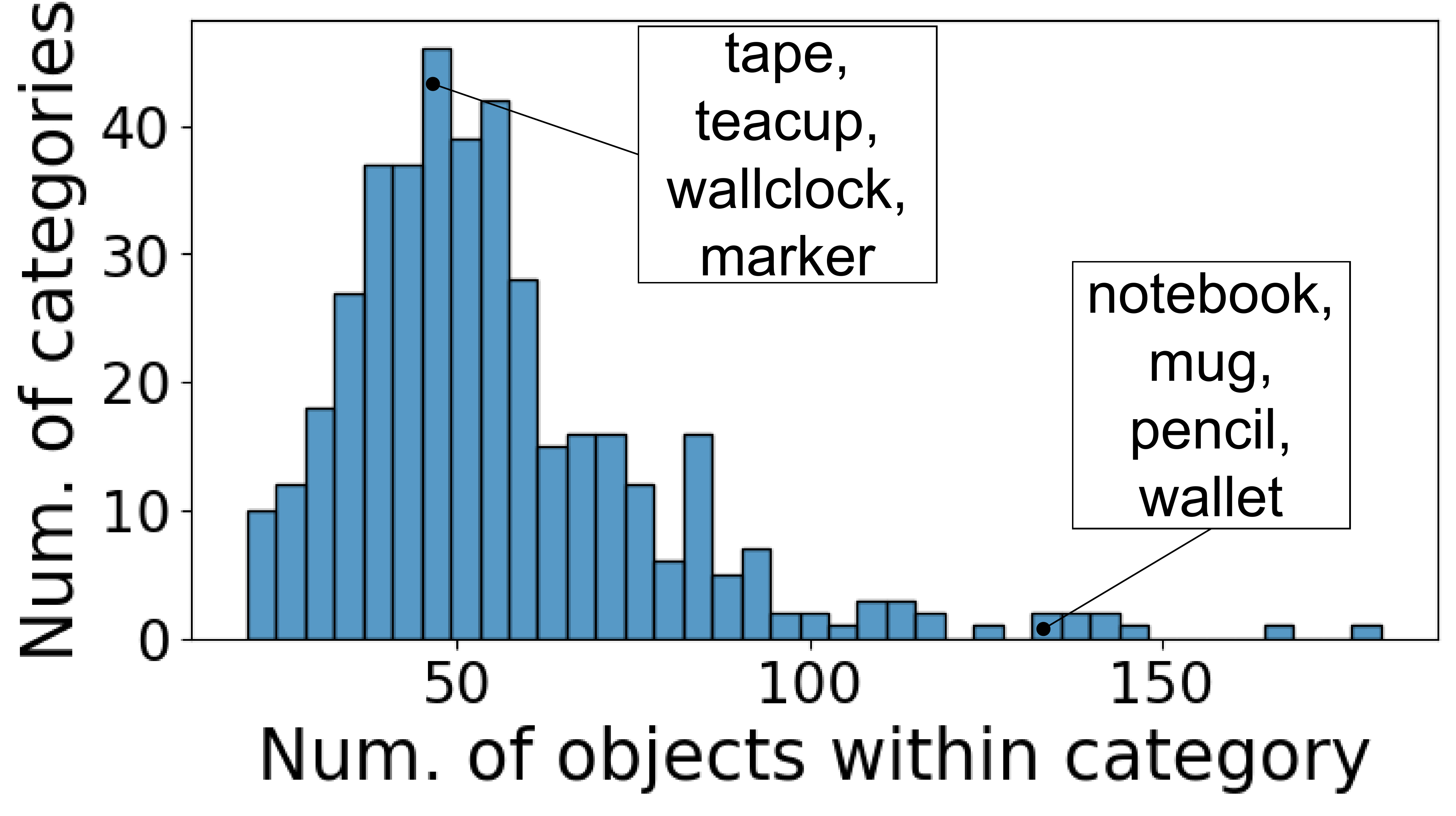}
    \caption{\label{fig:num-cats}Number of categories}
  \end{subfigure}
\begin{subfigure}{0.31\linewidth}
    \includegraphics[width=\linewidth]{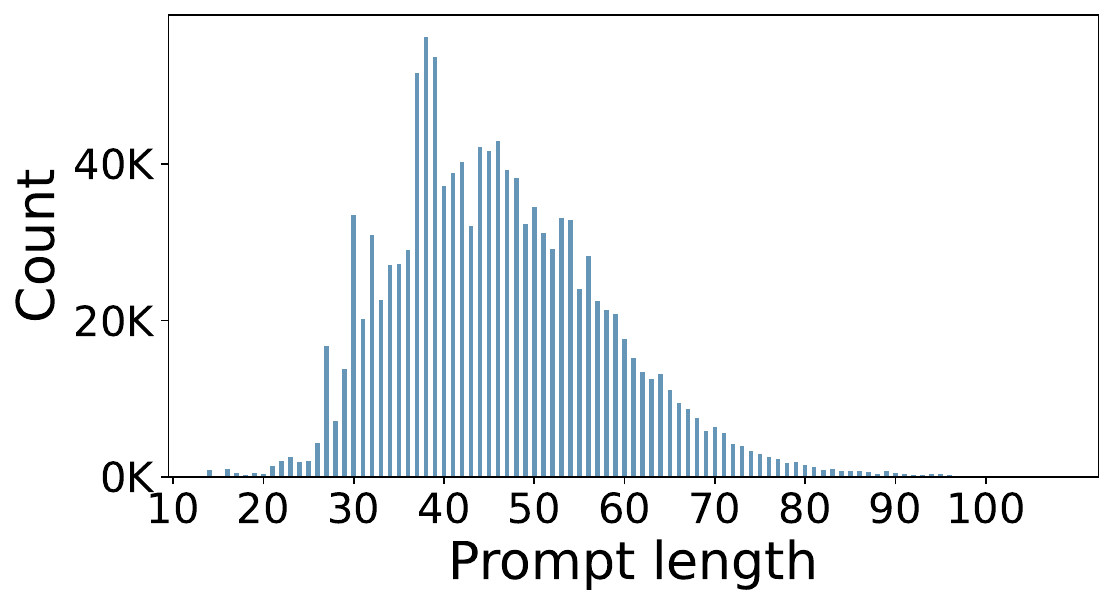}
    \caption{\label{fig:scene-length}Scene description length}
  \end{subfigure}
  \begin{subfigure}{0.38\linewidth}
    \includegraphics[width=\linewidth]{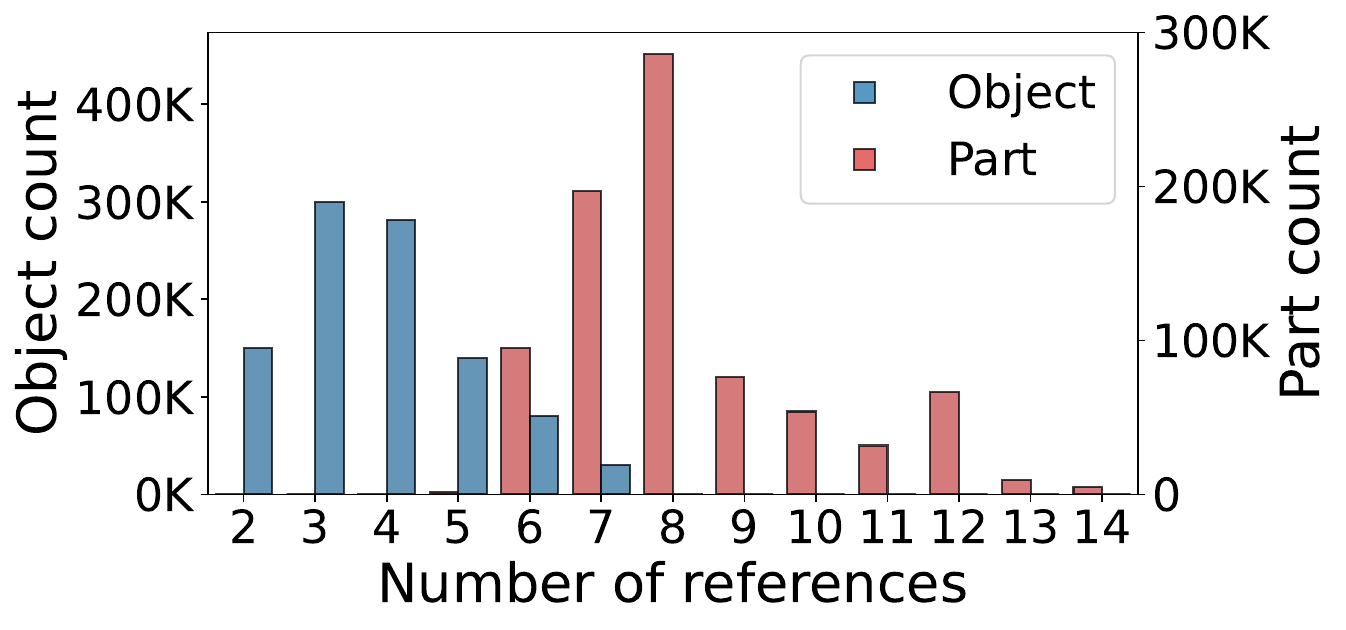}
    \caption{\label{fig:reference-count}Number of object/part references.}
  \end{subfigure}
  \vspace{-4ex}
  \caption{\textbf{Data statistics.} We analyze object categories and the use of natural language in the Grasp-Anything++ dataset.}
  \vspace{-4ex}
\end{figure*}

\textbf{Grasp Annotation.} Grasp poses are represented as 2D rectangles, consistent with prior research and practical compatibility with real-world parallel grippers~\cite{jiang2011efficient, depierre2018jacquard}. Utilizing a pretrained network~\cite{cao2023nbmod}, we annotate grasp poses based on part segmentation masks. Since potential inaccuracies in these candidate poses could occur, we follow the procedure as defined in~\cite{vuong2023grasp} to evaluate the quality of generated grasp poses to discard unreasonable grasp poses.

Specifically, grasp quality is evaluated through net torque $\mathcal{T} = \left(\tau_1+\tau_2\right) -RMg\Comma$ where resistance at contact points is $\tau_i = K\mu_sF\cos\alpha_i$. With constants such as $M$ (mass), $g$ (gravitational acceleration), $K$ (geometrical characteristics), $\mu_s$ (static friction coefficient), and $F$ (applied force), accurately determining $\mathcal{T}$ directly is challenging due to the physical difficulties in precisely measuring $M$, $K$, and $\mu_s$. Thus, we employ a surrogate measure, $\tilde{\mathcal{T}} = \dfrac{\cos\alpha_1 + \cos\alpha_2}{R}\Comma$ as an alternative. As a result in~\cite{chen1993finding}, antipodal grasps score higher on $\tilde{\mathcal{T}}$, indicating better quality. Consequently, grasps are evaluated based on $\tilde{\mathcal{T}}$, with positive values indicating positive grasps and others considered as negative.




\textbf{Post Processing.} Despite training on extensive datasets, Stable Diffusion~\cite{rombach2022high} may produce subpar content, commonly termed as hallucination\cite{huang2023survey} when generating images from the text prompts. To address this, we perform manual reviews to filter out such images, with qualitative examples in our figures. Our process includes checks at every stage to prevent duplicate or hallucinated content. However, manual inspection introduces biases, which we counter with guidelines focusing on abnormal structures or implausible gravity (Fig.~\ref{fig:hallucination}), aligns with approach in the literature~\cite{schuhmann2022laion}. 
\vspace{-2ex}
\begin{figure}[h]
    \centering
    \includegraphics[width=\linewidth]{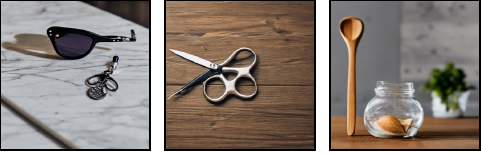}
    \vspace{-2ex}
    \caption{\textbf{Failure image generation cases.} Images generated by Stable Diffusion~\cite{rombach2022high} may exhibit hallucinatory artifacts such as sunglasses lacking a lens, scissors with an anomalous structure, and a spoon not resting properly on a table.}
    \label{fig:hallucination}
\end{figure}
\vspace{-2ex}

Additionally, ChatGPT-generated scene prompts often duplicate~\cite{liu2023summary}. To address this, we use duplication checking, filtering out identical prompts with BERTScore~\cite{zhang2019bertscore}, which assesses sentence similarity through cosine similarities of token embeddings. We remove sentences with a BERTScore above 0.85 as in prior study~\cite{wehnert2021legal}.

\subsection{Data Statistics}\label{subsec: data-stat}

\textbf{Number of Categories.} To evaluate the object category diversity, we apply a methodology akin to that in~\cite{deitke2023objaverse}. Utilizing 300 categories from LVIS dataset~\cite{gupta2019lvis}, we employ a pretrained model~\cite{zhong2022regionclip} to identify 300 candidate objects from our dataset for each category. We then curate a subset comprising 90,000 objects, refining it by excluding items that do not align semantically with their designated categories. A category is considered significant if it has more than 40 objects. Fig.~\ref{fig:num-cats} shows the results. Overall, our dataset spans over $236$ categories from LVIS dataset, indicating a notable degree of object diversity in our dataset.

\textbf{Scene Descriptions.} Fig.~\ref{fig:scene-length} shows the distribution of scene descriptions based on sentence length. The analysis reveals a wide range of sentence lengths, spanning from 10 words to 100 words per sentence. On average, each scene description consists of approximately 54 words, indicative of detailed and descriptive sentences. These scene descriptions correspond to sets of grasp instructions. Fig.~\ref{fig:reference-count} further shows the objects and object parts in scene descriptions.




\textbf{Diversity Analysis.} We assess the diversity of occlusion and lighting conditions in the dataset. Regarding occlusion, we use a pretrained YOLOv5 model to identify objects within images. The results indicate that $93.8\%$ of images have a substantial overlap of five or more bounding boxes, which suggests a diverse range of occlusion within the Grasp-Anything++ dataset. Regarding lighting conditions, we convert images to YCbCr to analyze Y channel (luminance) and find that GraspL1M has the most diverse lighting conditions, identifying by the lowest Gini coefficient (a metric to measure the inequality of a distribution) of $0.26$, compared to VMRD~\cite{zhang2019roi} ($0.31$), OCID-grasp~\cite{ainetter2021end} ($0.32$), Cornell~\cite{jiang2011efficient} ($0.62$), Jacquard~\cite{depierre2018jacquard} ($0.91$).
\section{Language-driven Grasp Detection}
\textbf{Motivation.} The use of diffusion model for language-driven grasp detection is motivated by its efficiency in various generative tasks~\cite{ho2020denoising, yang2022diffusion,le2023controllable,liu2022structdiffusion, croitoru2023diffusion}. Conditional generation, such as our language-driven grasp detection task, aligns seamlessly with diffusion models' capabilities~\cite{ho2022video}. 
Moreover, language-driven grasp detection represents a fine-grained problem in which the outputs strongly depend on the text input~\cite{asif2018graspnet}. For example, ``\textit{grasp the steak knife}" and ``\textit{grasp the kraft knife}" refer to two different objects on the image.
To this end, we propose using \textit{contrastive loss with diffusion model} to tackle this task, as contrastive learning is a popular solution for fine-grained tasks~\cite{bukchin2021fine,do2022fine, yang2023zero}.

\subsection{Constrastive Loss for Diffusion Model}
We represent the target grasp pose as $\mathbf{x}_0$ in the diffusion model. The objective of our diffusion process of language-driven grasp detection involves denoising from a noisy state $\mathbf{x}_T$ to the original grasp pose $\mathbf{x}_0$, conditioned on the input image and grasp instruction represented by $y$.

In a diffusion process~\cite{ho2020denoising}, assume that $q(\mathbf{x}_{1:T}\vert\mathbf{x}_0)$ is the forward process and we parameterize the reverse process by $p_{\theta}(\mathbf{x}_{0:T})$. The conditional diffusion process~\cite{dhariwal2021diffusion} assumes $\hat{q}$ is the forward process but with the inclusion of a condition $y$. The goal of the reverse process is to optimize the variational bound on negative log likelihood~\cite{ho2020denoising} 

\begin{equation}
\label{eq: log-likelihood-1}
    \mathcal{L} = \mathbb{E}\left[-\log p_{\theta}(\mathbf{x}_T) - \sum_{t\geq 1}\log\frac{p_{\theta}(\mathbf{x}_{t-1}\vert\mathbf{x}_t)}{q(\mathbf{x}_t\vert\mathbf{x}_{t-1})}\right]\FullStop
\end{equation}

We prove in the Appendix that

\begin{equation}
\label{eq: log-likelihood-4}
\begin{gathered}
\mathcal{L}=\mathbb{E}\bigg[\underbrace{C-\log p_\theta(\mathbf{x}_T)}_\textrm{Constant} + \underbrace{\log\hat{q}(\mathbf{x}_0\vert\mathbf{x}_T, y)}_\textrm{Contrastive} + \\ \underbrace{\sum_{t>1}D_{\text{KL}}(\hat{q}(\mathbf{x}_{t-1}\vert\mathbf{x}_t, \mathbf{x}_0, y)\|p_\theta(\mathbf{x}_{t-1}\vert\mathbf{x}_t))-\log p_\theta(\mathbf{x}_0\vert\mathbf{x}_1, y)}_\textrm{Denoising score}\bigg]\FullStop
\end{gathered}
\end{equation}

The terms $D_{\text{KL}}(\hat{q}(\mathbf{x}_{t-1}\vert\mathbf{x}_t, \mathbf{x}_0, y)\|p_\theta(\mathbf{x}_{t-1}\vert\mathbf{x}_t))$ and $\log p_\theta(\mathbf{x}_0\vert\mathbf{x}_1, y)$ in Equation~\ref{eq: log-likelihood-4} are similar to the concept of denoising score used in~\cite{ho2020denoising}. Thus, we can represent the quantity $\sum_{t>1}D_{\text{KL}}(\hat{q}(\mathbf{x}_{t-1}\vert\mathbf{x}_t, \mathbf{x}_0, y)\|p_\theta(\mathbf{x}_{t-1}\vert\mathbf{x}_t))-\log p_\theta(\mathbf{x}_0\vert\mathbf{x}_1, y)$ by the loss utilized in~\cite{tevet2022human, tseng2023edge}

\begin{equation}
    \label{eq: loss-denoising}
    \mathcal{L}_{\rm{diffusion}} = \mathbb{E}_{\mathbf{x}_0\sim q(\mathbf{x}_0\vert y), t\sim[1, T]}\left[\mathbf{x}_0-f(\mathbf{x}_{t+1}, t+1, \tilde{\mathbf{x}}_0)\right]^2\FullStop
\end{equation}

Since $q(\mathbf{x}_T\vert\cdot)$ is equivalent to an isotropic Gaussian distribution as $T\rightarrow+\infty$, the estimation quantity $\log p_\theta(\mathbf{x}_T)$ converges to a constant when $\theta\rightarrow\theta^\ast$. Therefore, we can ignore the first term of Equation~\ref{eq: log-likelihood-4}.

Finally, the term ${\hat{q}(\mathbf{x}_0\vert\mathbf{x}_T, y)}$ of Equation~\ref{eq: log-likelihood-4} provides more information about the relation between $\mathbf{x}_T$ and $\mathbf{x}_0$. As this term is intractable~\cite{sohl2015deep}, we parameterize ${\hat{q}(\mathbf{x}_0\vert\mathbf{x}_T, y)}$ as $p_{\psi, y}(\mathbf{x}_0, \mathbf{x}_T)$. This estimation resembles the noise-contrastive estimation~\cite{gutmann2010noise}, where the $\mathbf{x}_T$, $\mathbf{x}_0$ can be considered as a pair of contrastive estimation and $\psi$ can be estimated by a contrastive loss. 

In~\cite{vuong2023language, dhariwal2021diffusion, tevet2022human}, the authors indicate that predicting $\mathbf{x}_0$ is often infeasible but predicting an estimation $\tilde{\mathbf{x}}_0$ is tractable and can be used as a `pseudo' estimation of $\mathbf{x}_0$. We denote $\tilde{\mathbf{x}}_0$ as an estimation of $\mathbf{x}_0$. The loss term ${\hat{q}(\mathbf{x}_0\vert\mathbf{x}_T, y)}$ can be approximated by using the following contrastive loss

\begin{equation}
\label{eq: contrastive-loss}
    \mathcal{L}_{\rm{contrastive}} = \max \left(0,\left\Vert \frac{\sqrt{\overline{\alpha_T}}\tilde{\mathbf{x}}_0-\mathbf{x}_T}{\sqrt{1-\overline{\alpha_T}}}\right\Vert_2^2-M\right),
\end{equation}
where $M$ is the number of dimension of $\mathbf{x}_0$, and $\alpha_t$ is the variance schedule at timestep $t$ ($t=\overline{1;T}$). 

\begin{proposition}
\label{prop: prop_1} 
Suppose that $\tilde{\mathbf{x}}_0,\,\mathbf{x}_0$ and $\epsilon$ are independent, and that
\[
\left\Vert \frac{\sqrt{\overline{\alpha_T}}\tilde{\mathbf{x}}_0-\mathbf{x}_T}{\sqrt{1-\overline{\alpha_T}}} \right\Vert_2^2 \geq M\FullStop
\]
Then there exists $C>0$ such that: for arbitrary $\delta>0$, if $\mathcal{L}_{\rm{contrastive}}<\delta$, then 
\[
\mathbb{E}\left[\Vert \tilde{\mathbf{x}}_0-\mathbf{x}_0 \Vert_2^2\right]<C\delta\FullStop
\]
\end{proposition}
\begin{proof}
See Supplementary Material.
\end{proof}

\begin{remark}

\begin{itemize}[leftmargin=*]
	\item[] Proposition~\ref{prop: prop_1} suggests that if the contrastive loss $\mathcal{L}_{\rm{contrastive}}$ tends to zero, then the prediction $\tilde{\mathbf{x}}_0$ will approach the ground truth $\mathbf{x}_0$.
 
\end{itemize}
\end{remark}

\begin{remark}
\begin{itemize}[leftmargin=*]
	\item[] The condition $\left\Vert \frac{\sqrt{\overline{\alpha_T}}\tilde{\mathbf{x}}_0-\mathbf{x}_T}{\sqrt{1-\overline{\alpha_T}}} \right\Vert_2^2 \geq M$ is suitable for our language-driven grasp detection task as $\tilde{\mathbf{x}}_0$ and $\mathbf{x}_T$ are two contrastive quantities, therefore, we can assume there is a minimum distance between $\tilde{\mathbf{x}}_0$ and $\mathbf{x}_T$. In addition, in the proof of Proposition~\ref{prop: prop_1}, we see that $\mathbb{E} \left[ \left\Vert \frac{\sqrt{\overline{\alpha_T}}\tilde{\mathbf{x}}_0-\mathbf{x}_T}{\sqrt{1-\overline{\alpha_T}}} \right\Vert_2^2 - M \right] = \beta^2 \mathbb{E} \left[ \Vert \tilde{\mathbf{x}}_0  - \mathbf{x}_0 \Vert_2^2 \right]$, which is always nonnegative. Therefore, it is both theoretically and experimentally reasonable to add this assumption. 
 \end{itemize}
\end{remark}

\begin{figure}[!ht]
    \centering
    \includegraphics[width=1.02\linewidth , trim={0.8cm 0.3cm 1cm 0.3cm}]{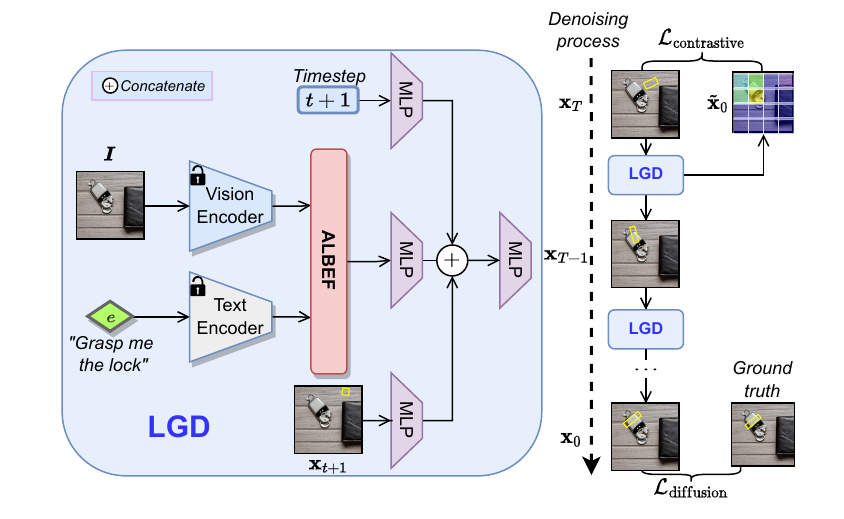}
    \caption{\textbf{Language-drive Grasp Detection (LGD) network.} We present the network architecture (left) and the proposed training objectives of the denoising process (right).}
    \label{fig:method-overview}
    \vspace{-2ex}
\end{figure}

\subsection{Language-driven Grasp Detection Network}




\begin{figure*}
    \includegraphics[width=\linewidth, trim={0 0 0 0.2cm}, clip]{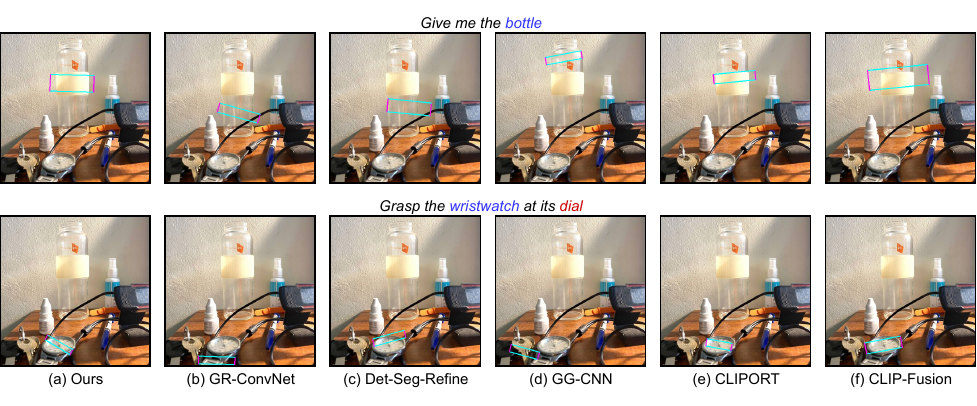}
    \vspace{-6ex}\caption{\label{fig:language-driven-comparison} \textbf{Language-driven grasp detection results visualization.}}
    \vspace{-4ex}
\end{figure*}

\textbf{Network.} Our network operates on two conditions: an image denoted as $\textit{\textbf{I}}$ and a corresponding text prompt represented as $e$. To process these conditions, we employ a vision encoder to extract visual features from $\textit{\textbf{I}}$ and a text encoder to derive textual embeddings from $e$. The resulting feature vectors, denoted as $\textit{\textbf{I}}^\prime$ and $e^\prime$, are subsequently subjected to a fusion module, ALBEF~\cite{li2021align}. We leverage the attention mask generated by the ALBEF fusion module as the estimation $\tilde{\mathbf{x}}_0$ of $\mathbf{x}_0$. Next, we aggregate three elements: the estimation region $\tilde{\mathbf{x}}_0$, the grasp pose at the current timestep $\mathbf{x}_{t+1}$, and the timestep $t+1$. These inputs are combined using MLP layers, similar to the approach outlined in~\cite{tevet2022human}. Specifically, the output operation can be expressed as: $\mathbf{x}_t = f(\mathbf{x}_{t+1}, t+1, \tilde{\mathbf{x}}_0)$, where function $f$ encompasses a composition of multiple MLP layers. Additional specifics regarding these universal MLP layers are provided in the Supplementary Material.

\textbf{Training Objective.} In our context, conditioned grasp detection models the distribution $p(\mathbf{x}_0|y)$ as the reversed diffusion process of gradually cleaning $\mathbf{x}_{t+1}$. Instead of predicting $\mathbf{x}_{t}$ as formulated by~\cite{ho2020denoising},
we follow Ramesh \textit{et al.}~\cite{ramesh2022hierarchical} and predict the signal itself, i.e., $\mathbf{x}_t = f(\mathbf{x}_{t+1}, t+1, \tilde{\mathbf{x}}_0)$ with the simple
objective~\cite{ho2020denoising}. To this end, we utilize the contrastive loss as in Equation~\ref{eq: contrastive-loss} to explicitly improve the learning objective of the denoising process:

\begin{equation}
    \label{eq: training-objective}
 \mathcal{L}_{\rm{total}} = \mathcal{L}_{\rm{contrastive}} + \mathcal{L}_{\rm{diffusion}}\FullStop  
\end{equation}


\graphicspath{./images}

\section{Experiments}
We conduct experiments to evaluate our proposed method and Grasp-Anything++ dataset using both the vision-based metrics and real robot experiments. We then demonstrate zero-shot grasp results and discuss the challenges and open questions for future works.


\subsection{Language-driven Grasp Detection Results}
\label{subsec:language-driven-grasp-detection}

\textbf{Baselines.} We compare our language-driven grasp detection method (LGD) with the linguistically supported versions of GR-CNN~\cite{kumra2020antipodal}, Det-Seg-Refine~\cite{ainetter2021end}, GG-CNN~\cite{morrison2018closing}, CLIPORT~\cite{shridhar2022cliport} and CLIP-Fusion~\cite{xu2023joint}. In all cases, we employ a pretrained CLIP~\cite{radford2021learning} or BERT~\cite{devlin2018bert} as the text embedding. The implementation details of all baselines can be found in our Supplementary Material.

\textbf{Setup.} 
To assess the generalization of all methods trained on Grasp-Anything++, we utilize the concept of base and new labels~\cite{zhou2022conditional} in zero-shot learning. We categorize LVIS labels from Section~\ref{subsec: data-stat} to form labels for our experiment. In particular, we select 70\% of these labels by frequency for `Base' and assign the remaining 30\% to `New'. We also use the harmonic mean (`H') to measure the overall success rates~\cite{zhou2022conditional}. Our primary evaluation metric is the success rate, defined similarly to~\cite{kumra2020antipodal}, necessitating an IoU score of the predicted grasp exceeding $25\%$ with the ground truth grasp and an offset angle less than $30^\circ$.


\begin{table}[t]
\centering
\renewcommand
\tabcolsep{4.5pt}
\hspace{1ex}
\begin{tabular}{@{}rccc@{}}
\toprule
Baseline & Seen & Unseen & H \cr 
\midrule
GR-ConvNet~\cite{kumra2020antipodal} + CLIP~\cite{radford2021learning}  & 0.37 & 0.18 & 0.24\\
Det-Seg-Refine~\cite{ainetter2021end} + CLIP~\cite{radford2021learning} & 0.30 & 0.15 & 0.20\\
GG-CNN~\cite{morrison2018closing} + CLIP~\cite{radford2021learning} & 0.12 & 0.08 & 0.10
\\
CLIPORT~\cite{shridhar2022cliport} & 0.36 & 0.26 & 0.29\\
CLIP-Fusion~\cite{xu2023joint} & 0.40 & 0.29 & 0.33\\ \midrule
LGD (ours) + BERT~\cite{devlin2018bert} & 0.44 & 0.38 & 0.41\\
LGD (ours) + CLIP~\cite{radford2021learning} & \textbf{0.48} & \textbf{0.42} & \textbf{0.45}\\

\bottomrule
\end{tabular}
\vspace{-1ex}
\caption{\textbf{Language-driven grasp detection results.} }
\label{table: language-driven grasp detection results}
\vspace{-3ex}
\end{table}

\textbf{Main Results.} Table~\ref{table: language-driven grasp detection results} shows the results of language-driven grasp detection on the Grasp-Anything++ dataset. The findings indicate a notable performance advantage of our LGD over other baseline approaches, with LGD outperforming the subsequent best-performing baselines (CLIP-Fusion) by margins of $0.14$ on Grasp-Anything++ dataset.

\begin{table}[h]
\centering
\renewcommand
\tabcolsep{4.5pt}
\hspace{1ex}
\begin{tabular}{@{}rccc@{}}
\toprule
Baseline & Seen & Unseen & H \cr 
\midrule
LGD w/o predicting $\tilde{\mathbf{x}}_0$ & 0.15 & 0.08 & 0.10\\
LGD w/o contrastive loss & 0.45 & 0.40 & 0.42 \\ 
LGD w contrastive loss & \textbf{0.48} & \textbf{0.42} & \textbf{0.45}\\ 
\bottomrule
\end{tabular}
\vspace{-1ex}
\caption{\textbf{Contrastive loss analysis.}}
\label{table:ablation-study}
\vspace{-2ex}
\end{table}

\begin{figure}[h]
    \centering
    \includegraphics[width=\linewidth]{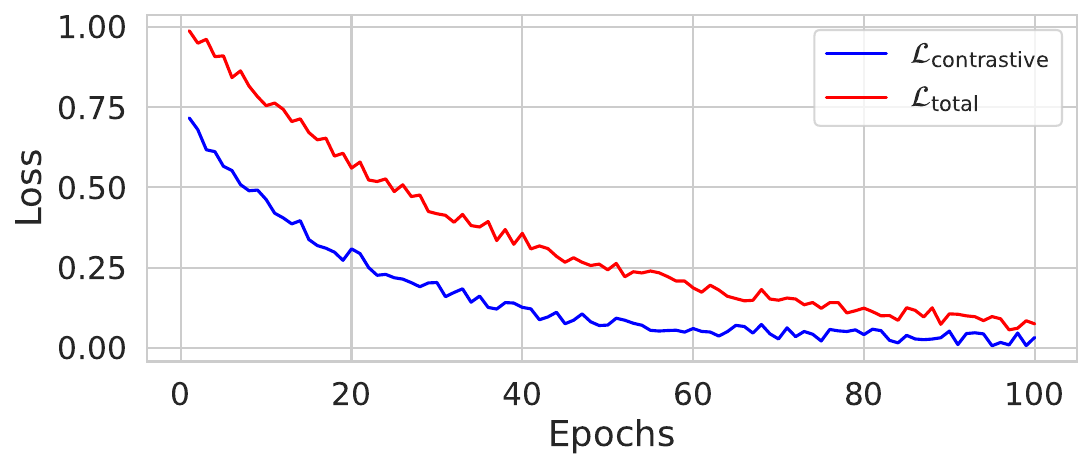}
    \vspace{-4ex}
    \caption{\textbf{Loss visualization.}}
    \label{fig:loss-visualization}
    \vspace{-2ex}
\end{figure}

\begin{figure}[!ht]
    \centering
    \includegraphics[width=\linewidth]{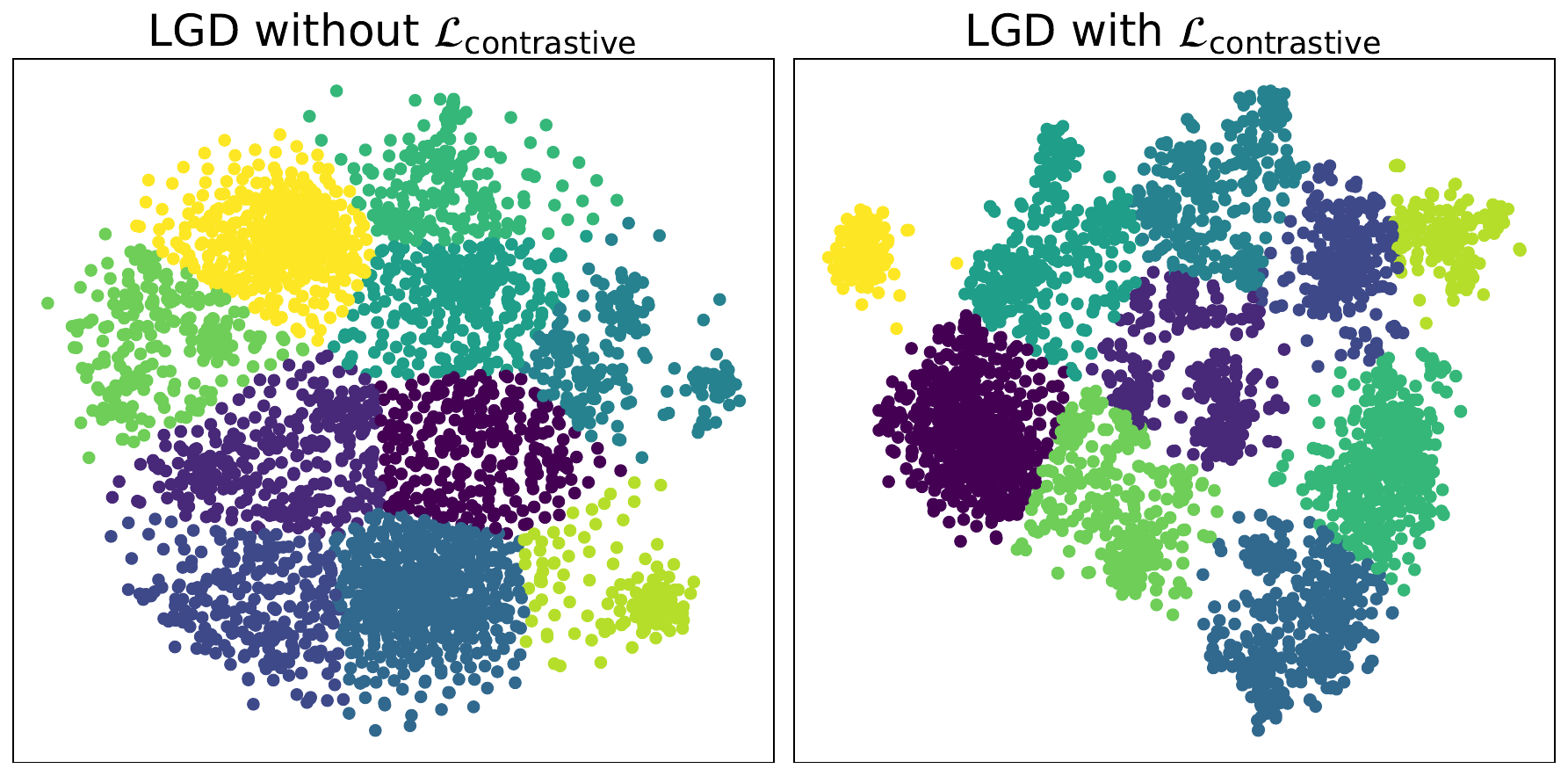}
    \vspace{-3ex}
    \caption{\textbf{t-SNE visualization.} We apply t-SNE to cluster the vision-and-language features with and without contrastive loss.}
    \label{fig: t-sne}
    \vspace{-3ex}
\end{figure}

\textbf{Contrastive Loss Analysis.} Table~\ref{table:ablation-study} presents the performance of LGD under varied configurations. The outcomes emphasize the substantial influence of the training objective (contrastive loss) and the importance of language instructions in enhancing LGD performance on both seen and unseen classes in the Grasp-Anything++ dataset. 

Fig.~\ref{fig:loss-visualization} shows the contrastive loss $\mathcal{L}_\text{contrastive}$ approaching towards $0$ during training, indicating the grasp pose estimation $\tilde{\mathbf{x}}_0$ aligns with the ground truth $\mathbf{x}_0$, as anticipated by Proposition~\ref{prop: prop_1}. The subsequent attention maps visualization in Fig.~\ref{fig:guiding-region} shows the attention region is meaningful and improves the results when employing our proposed contrastive loss compared to its absence. Moreover, we employ t-SNE for vision-and-language embedding visualization, as in~\cite{miao2023fedseg}, by processing $2,000$ samples from the Grasp-Anything++ dataset through the ALBEF module. The outcomes reveal that our contrastive loss facilitates better object classification, as evidenced in Fig.\ref{fig: t-sne} by clearer segregation of pixel embeddings across various semantic classes, underscoring contrastive loss's role in refining embeddings' differentiation for improved class distinctions.

\textbf{Qualitative Results.} Fig.~\ref{fig:language-driven-comparison} presents qualitative results of the language-driven grasp detection task, suggesting that our LGD method generates more semantically plausible than other baselines. Despite satisfactory performance, LGD occasionally predicts incorrect results, with a detailed analysis of these cases available in our Appendix.

\begin{figure}[ht]
    \centering
    \includegraphics[width=1.0\linewidth, trim={1cm 0.3cm 1cm 0}, clip]{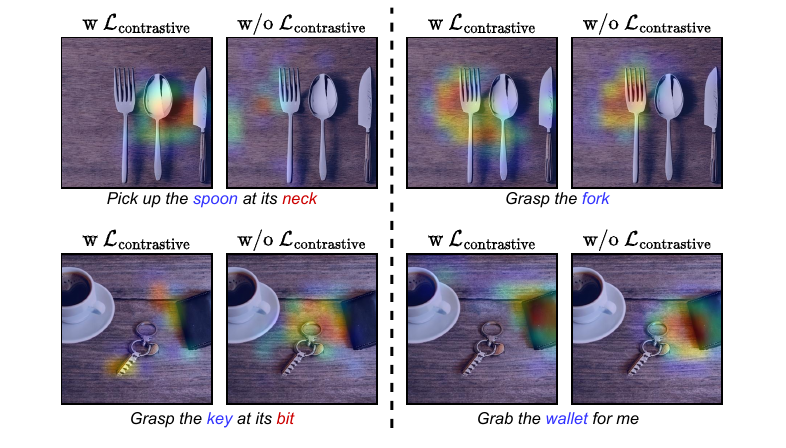}
    \caption{\textbf{Attention map visualization.} We compare the attention map when utilizing our proposed contrastive loss and when not.}
    \label{fig:guiding-region}
    \vspace{-2ex}
\end{figure}

\begin{table}[h]
    \centering
    \renewcommand
\tabcolsep{4pt}
\hspace{1ex}
    \begin{tabular}{@{}rcc@{}}
\toprule
Baseline & Single &  Cluttered\cr 
\midrule
GR-ConvNet~\cite{kumra2020antipodal} + CLIP~\cite{radford2021learning}  &0.33  & 0.30\\ 
Det-Seg-Refine~\cite{ainetter2021end} + CLIP~\cite{radford2021learning} &0.30  & 0.23\\ 
GG-CNN~\cite{morrison2018closing} + CLIP~\cite{radford2021learning} &0.10  & 0.07 \\
CLIPORT~\cite{shridhar2022cliport} &0.27 & 0.30 \\
CLIP-Fusion~\cite{xu2023joint} & 0.40 & 0.40 \\
LGD (ours) &  \textbf{0.43} & \textbf{0.42} \\
\bottomrule
\end{tabular}
\caption{\label{table: real-robot-language-driven} \textbf{Robotic language-driven grasp detection results.}}
\vspace{-2ex}
\end{table}

\textbf{Robotic Validation.} 
We provide quantitative results by integrating our language-driven grasp detection pipeline for a robotic grasping application with a \kuka robot. Using the RealSense D435i camera, the grasp pose inferred from approaches in Table~\ref{table: real-robot-language-driven} is transformed into the 6DoF grasp pose, similar to~\cite{kumra2020antipodal}. The optimization-based trajectory planner in~\cite{vu2023machine,beck2023singularity} is employed to execute the grasps. Experiments are conducted for two scenarios, i.e., the single object scenario and the cluttered scene scenario, of a set of $20$ real-world daily objects. In each scenario, we run $30$ experiments using baselines listed in Table~\ref{table: real-robot-language-driven} and a predefined grasping prompt corpus. The results exhibit that our LGD outperforms other baselines. Furthermore, although LGD is trained on our Grasp-Anything++ which is a solely synthesis dataset created by foundation models, it still shows reasonable results on real-world objects.


\begin{table*}[ht]
\centering
\renewcommand
\tabcolsep{4.5pt}
\hspace{1ex}
\vskip 0.1 in
\resizebox{\linewidth}{!}{
\begin{tabular}{@{}rccccccccccccccccc@{}}
\toprule
  & \multicolumn{3}{c}{\textbf{Grasp-Anything++ (ours)}} & \multicolumn{3}{c}{\textbf{Jacquard~\cite{depierre2018jacquard}}} & \multicolumn{3}{c}{\textbf{Cornell~\cite{jiang2011efficient}}} & \multicolumn{3}{c}{\textbf{VMRD~\cite{zhang2019roi}}} & \multicolumn{3}{c}{\textbf{OCID-grasp~\cite{ainetter2021end}}}\\
\cmidrule(lr){2-4}\cmidrule(lr){5-7}\cmidrule(lr){8-10}\cmidrule(lr){11-13}\cmidrule(lr){14-16}
Baseline & Base &  New & H & Base &  New & H & Base &  New & H & Base &  New & H & Base &  New & H\cr 
\midrule
GR-ConvNet~\cite{kumra2020antipodal} & 0.71 & 0.59 & 0.64 & 0.88 & 0.66 & 0.75 & 0.98 & 0.74 & 0.84 & 0.77 & 0.64 & 0.70 & 0.86 & 0.67 & 0.75 \\

Det-Seg-Refine~\cite{ainetter2021end} & 0.62 & 0.57 & 0.59 & 0.86 &  0.60 & 0.71 & \textbf{0.99} &  \textbf{0.76} & \textbf{0.86} & 0.75 & 0.60 & 0.66 & 0.80 & 0.62 & 0.70\\

GG-CNN~\cite{morrison2018closing} & 0.68 & 0.57 & 0.62 & 0.78 & 0.56 & 0.65 & 0.96 & 0.75 & 0.84 & 0.69 & 0.53 & 0.59 & 0.71 & 0.63 & 0.67 \\

LGD (no text) (ours) & \textbf{0.74} & \textbf{0.63} & \textbf{0.68} & \textbf{0.89} & \textbf{0.69} & \textbf{0.77} & 0.97 & \textbf{0.76} & 0.85 & \textbf{0.79} & \textbf{0.66} & \textbf{0.72} & \textbf{0.88} & \textbf{0.68} & \textbf{0.76} \\

\bottomrule
\end{tabular}
}
\vspace{-1ex}
\caption{\label{table: base-to-new-grasp-detection} \textbf{Base-to-new zero-shot grasp detection results.}}
\vspace{-3ex}
\end{table*}

\subsection{Zero-shot Grasp Detection}
Our proposed Grasp-Anything++ is a large-scale dataset. Apart from the language-driven grasp detection task, we believe it can be used for other purposes. In this experiment, we seek to answer the question: Can Grasp-Anything++ be useful in the traditional grasp detection task without text? Consequently, we verify our Grasp-Anything++ and LGD (no text) with other existing datasets and grasping methods.

\textbf{Setup.} We setup an LGD (no text) version, and other state-of-the-art grasp detection methods GR-ConvNet~\cite{kumra2020antipodal}, Det-Seg-Refine~\cite{ainetter2021end}, GG-CNN~\cite{morrison2018closing}. We use five datasets: our Grasp-Anything++, Jacquard~\cite{depierre2018jacquard}, Cornell~\cite{jiang2011efficient}, VMRD~\cite{zhang2019roi}, and OCID-grasp~\cite{ainetter2021end} in this experiment. 


\textbf{Zero-shot Results.} Table~\ref{table: base-to-new-grasp-detection} summarizes the base-to-new grasp detection results on five datasets. Overall, the performance of LGD even without the language branch is better than other baselines across all datasets. Furthermore, this table also shows that our Grasp-Anything++ dataset is more challenging to train as the detection results are lower than related datasets using the same approaches due to the greater coverage of unseen objects in the testing phase.


\begin{table}[htp]
\vspace{-1.5ex}
\centering
\renewcommand
\tabcolsep{4.5pt}
\hspace{1ex}
\vskip 0.1 in
\resizebox{\linewidth}{!}{
\begin{tabular}{@{}rccccc@{}}
\toprule
 \diagbox{Train}{Test} & \textbf{Jacquard} & \textbf{Cornell} & \textbf{VMRD} & \textbf{OCID-grasp} & \textbf{Grasp-Anything++}\\
\midrule Jacquard~\cite{depierre2018jacquard} & \underline{0.87} & 0.51 & 0.13 & 0.21 & \textbf{0.17} \\

Cornell~\cite{jiang2011efficient} & 0.07 & \underline{0.98} & 0.20 & 0.12 & 0.13 \\

VMRD~\cite{zhang2019roi} & 0.06 & 0.21 & \underline{0.79} & 0.11 & 0.10 \\

OCID-grasp~\cite{ainetter2021end} & 0.09 & 0.12 & 0.20 & \underline{0.74} & 0.11 \\ \midrule

Grasp-Anything++ (ours) & \textbf{0.41} & \textbf{0.63} & \textbf{0.30} & \textbf{0.39} & \underline{0.65} \\
\bottomrule
\end{tabular}}
\vspace{0ex}
\caption{\label{table: cross-dataset} Cross-dataset grasp detection results.}
\vspace{-2ex}
\end{table}

\textbf{Cross-dataset Evaluation.} To further verify the usefulness of our Grasp-Anything++ dataset, we conduct the cross-dataset validation in Table~\ref{table: cross-dataset}. We use the GR-ConvNet~\cite{kumra2020antipodal} to reuse its results on existing grasp datasets. GR-ConvNet is trained on a dataset (row) and evaluated on another dataset (column). For example, training on Jacquard and testing on Cornell yields an accuracy of $0.51$. Notably, training with our dataset improves performance by approximately $10-33\%$ compared to other datasets. 

\textbf{In the wild grasp detection.} Fig.~\ref{fig:in-the-wild} shows visualization results using LGD (no text) trained on our Grasp-Anything++ dataset on random internet images and other datasets images. We can see that the detected grasp poses are adequate in quality and quantity. This demonstrates that although our Grasp-Anything++ is fully created by foundation models without having any real images, models trained on our Grasp-Anything++ dataset still generalize well on real-world images.  

\begin{figure}[ht]
    \centering
    \includegraphics[width=\linewidth]{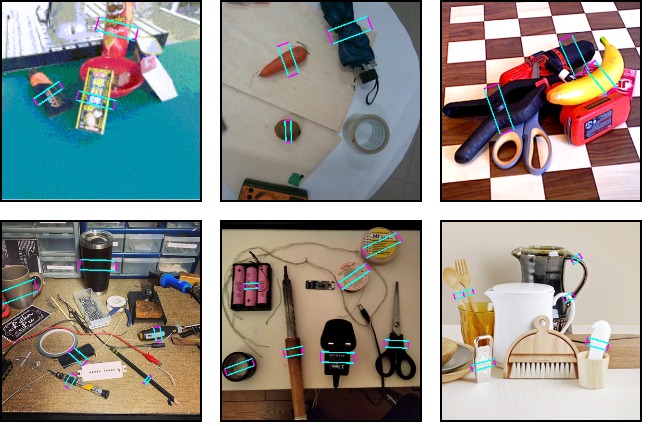}
    \vspace{-4ex}
\caption{\textbf{In the wild grasp detection}. Top row images are from GraspNet~\cite{fang2020graspnet}, YCB-Video~\cite{xiang2017posecnn}, NBMOD~\cite{cao2023nbmod} datasets; bottom row shows internet images.}
    \label{fig:in-the-wild}
    \vspace{-4ex}
\end{figure}

\subsection{Discussion}
Our experiments indicate that Grasp-Anything++ can serve as a foundation dataset for both language-driven and traditional grasp detection tasks. However, there are certain limitations. First, our dataset lacks depth images for directly being applied to robotic applications~\cite{newbury2023deep}. Second, we remark that the creation of our dataset is time-consuming and relies on access to the ChatGPT API. Fortunately, future research can reuse our provided assets (images, prompts, etc.) without starting from scratch. Furthermore, our experiments show that adding language to the grasp detection task (Table~\ref{table: language-driven grasp detection results}) poses a more challenging problem compared to standard grasp detection task (Table~\ref{table: base-to-new-grasp-detection}). 



We see several interesting future research directions. First, future work could investigate the use of text or image-to-3D models~\cite{xu2023neurallift} or image-to-depth~\cite{ranftl2021vision} and reuse our dataset's prompts and images to construct 3D language-driven grasp datasets. Additionally, beyond linguistic grasp instruction adherence, our dataset holds potential for varied applications, including scene understanding~\cite{xu2023human} and scene generation~\cite{arad2021compositional}, hallucination analysis~\cite{huang2023survey}, and human-robot interaction~\cite{cascante2022simvqa}.

\section{Conclusion}
We introduce Grasp-Anything++, a large-scale dataset with 1M images and 10M grasp prompts for language-driven grasp detection tasks. We propose LGD, a diffusion-based method to tackle the language-driven grasp detection task. Our diffusion model employs a contrastive training objective, which explicitly contributes to the denoising process. Empirically, we have shown that Grasp-Anything++ serves as a foundation grasp detection dataset. Finally, our LGD improves the performance of other baselines, and the real-world robotic experiments further validate the effectiveness of our dataset and approach.



{\small
\bibliographystyle{ieee_fullname}
\bibliography{egbib}
}

\newpage
\appendix

\section{Theoretical Findings}\label{sec: theoretical-findings}
In this section, we first show the derivation of Equation 4 in our main paper. We then show the proof of Proposition 1 in the main paper.

\subsection{Derivation of Equation 4} 
It was indicated in~\cite{dhariwal2021diffusion} that $q(\mathbf{x}_t\vert\mathbf{x}_{t-1})=\hat{q}(\mathbf{x}_t\vert\mathbf{x}_{t-1}, y)$, therefore, the loss in Equation 3 in our main paper can be written as

\begin{equation}
\label{eq: log-likelihood-2}
    \mathcal{L} = \mathbb{E}\left[-\log p_{\theta}(\mathbf{x}_T) - \sum_{t\geq 1}\log\frac{p_{\theta}(\mathbf{x}_{t-1}\vert\mathbf{x}_t)}{\hat{q}(\mathbf{x}_t\vert\mathbf{x}_{t-1}, y)}\right]\FullStop
\end{equation}

Using Bayes' Theorem, we can further derive the term $\hat{q}(\mathbf{x}_t\vert\mathbf{x}_{t-1}, y)$ of Equation~\ref{eq: log-likelihood-2} as follows

\begin{align*}
    \hat{q}&(\mathbf{x}_t\vert\mathbf{x}_{t-1}, y) = \frac{\hat{q}(\mathbf{x}_t, \mathbf{x}_{t-1}, y)}{\hat{q}(\mathbf{x}_{t-1}, y)} \\
    &= \frac{\hat{q}(\mathbf{x}_t, \mathbf{x}_{t-1}, y)}{\hat{q}(\mathbf{x}_t, \mathbf{x}_{t-1}, \mathbf{x}_0, y)}\frac{\hat{q}(\mathbf{x}_t, \mathbf{x}_{t-1}, \mathbf{x}_0, y)}{\hat{q}(\mathbf{x}_{t-1}, y)} \\
    &= \frac{1}{\hat{q}(\mathbf{x}_0\vert \mathbf{x}_{t-1}, \mathbf{x}_t, y)}\frac{\hat{q}(\mathbf{x}_t, \mathbf{x}_{t-1}, \mathbf{x}_0, y)}{\hat{q}(\mathbf{x}_t,\mathbf{x}_0, y)}\frac{\hat{q}(\mathbf{x}_t,\mathbf{x}_0, y)}{\hat{q}(\mathbf{x}_{t-1}, y)} \\
    &=\frac{\hat{q}(\mathbf{x}_{t-1}\vert\mathbf{x}_t, \mathbf{x}_0, y)}{\hat{q}(\mathbf{x}_0\vert \mathbf{x}_{t-1}, \mathbf{x}_t, y)}\frac{\hat{q}(\mathbf{x}_t,\mathbf{x}_0, y)}{\hat{q}(\mathbf{x}_{t-1}, y)}  \\
    &= \frac{\hat{q}(\mathbf{x}_{t-1}\vert\mathbf{x}_t, \mathbf{x}_0, y)}{\hat{q}(\mathbf{x}_0\vert \mathbf{x}_{t-1}, \mathbf{x}_t, y)}\frac{\hat{q}(\mathbf{x}_t,\mathbf{x}_0, y)}{\hat{q}(\mathbf{x}_{t-1}, \mathbf{x}_0, y)}\frac{\hat{q}(\mathbf{x}_{t-1}, \mathbf{x}_0, y)}{\hat{q}(\mathbf{x}_{t-1}, y)} \\
    &= \frac{\hat{q}(\mathbf{x}_{t-1}\vert\mathbf{x}_t, \mathbf{x}_0, y)}{\hat{q}(\mathbf{x}_0\vert \mathbf{x}_{t-1}, \mathbf{x}_t, y)}\frac{\hat{q}(\mathbf{x}_t\vert \mathbf{x}_0, y)}{\hat{q}(\mathbf{x}_{t-1}\vert \mathbf{x}_0, y)}\hat{q}(\mathbf{x}_0\vert \mathbf{x}_{t-1}, y)\FullStop
\end{align*}

Follow by Ho \textit{et al.}~\cite{ho2020denoising}, we can assume that $\hat{q}(\mathbf{x}_0\vert\mathbf{x}_{t-1}, \mathbf{x}_t, y) = \hat{q}(\mathbf{x}_0\vert\mathbf{x}_{t-1}, y)$ due to the Markov chain. Thus, $\hat{q}(\mathbf{x}_t\vert\mathbf{x}_{t-1}, y)$ can be further derived as follows

\begin{equation}
    \label{eq: simplify-form-1}
    \hat{q}(\mathbf{x}_t\vert\mathbf{x}_{t-1}, y) = \hat{q}(\mathbf{x}_{t-1}\vert\mathbf{x}_t, \mathbf{x}_0, y)\frac{\hat{q}(\mathbf{x}_t\vert \mathbf{x}_0, y)}{\hat{q}(\mathbf{x}_{t-1}\vert \mathbf{x}_0, y)}\FullStop
\end{equation}

From Equation~\ref{eq: log-likelihood-2} and Equation~\ref{eq: simplify-form-1}, we can express the negative log likelihood loss as follows

\begin{equation}
\label{eq: log-likelihood-3}
\begin{gathered}
L=\mathbb{E}\bigg[-\log\frac{p_\theta(\mathbf{x}_T)}{\hat{q}(\mathbf{x}_T\vert\mathbf{x}_0, y)}-\sum_{t>1}\log\frac{p_\theta(\mathbf{x}_{t-1}\vert\mathbf{x}_t)}{\hat{q}(\mathbf{x}_{t-1}\vert\mathbf{x}_t, \mathbf{x}_0, y)} \\
- \log p_\theta(\mathbf{x}_0\vert\mathbf{x}_1, y)\bigg]\FullStop
\end{gathered}
\end{equation}

By using Bayes' Theorem again, we can formulate $\hat{q}(\mathbf{x}_T\vert\mathbf{x}_0, y)$ of Equation~\ref{eq: log-likelihood-3} as follows

\begin{align}
    \hat{q}(\mathbf{x}_T\vert\mathbf{x}_0, y) &= \frac{\hat{q}(\mathbf{x}_T, \mathbf{x}_0, y)}{\hat{q}(\mathbf{x}_0, y)} \nonumber\\
    &=\frac{\hat{q}(\mathbf{x}_T, \mathbf{x}_0, y)}{\hat{q}(\mathbf{x}_0)}\frac{\hat{q}(\mathbf{x}_0)}{\hat{q}(\mathbf{x}_0, y)} \nonumber\\ &=\frac{\hat{q}(\mathbf{x}_0\vert\mathbf{x}_T, y)}{\hat{q}(y\vert\mathbf{x}_0)}\FullStop
\end{align}


Since $\hat{q}(y\vert\mathbf{x}_0)$ is known labels per sample~\cite{dhariwal2021diffusion}, thus, can be treated as a constant $C$. We conclude with the final derivation of Equation~\ref{eq: log-likelihood-3} by

\begin{align}
    L&=\mathbb{E}\bigg[-\log\frac{p_\theta(\mathbf{x}_T)}{\hat{q}(\mathbf{x}_0\vert\mathbf{x}_T, y)}+\log\hat{q}(y\vert\mathbf{x}_0) \nonumber\\
    &\qquad -\sum_{t>1}\log\frac{p_\theta(\mathbf{x}_{t-1}\vert\mathbf{x}_t)}{\hat{q}(\mathbf{x}_{t-1}\vert\mathbf{x}_t, \mathbf{x}_0, y)} -\log p_\theta(\mathbf{x}_0\vert\mathbf{x}_1, y)\bigg] \nonumber\\
    &= \mathbb{E}\bigg[C-\log p_\theta(\mathbf{x}_T) + \log\hat{q}(\mathbf{x}_0\vert\mathbf{x}_T, y) + \nonumber\\ 
    &\qquad \sum_{t>1}D_{\text{KL}}(\hat{q}(\mathbf{x}_{t-1}\vert\mathbf{x}_t, \mathbf{x}_0, y)\|p_\theta(\mathbf{x}_{t-1}\vert\mathbf{x}_t)) \nonumber\\
    &\qquad \qquad -\log p_\theta(\mathbf{x}_0\vert\mathbf{x}_1, y)\bigg]\FullStop
\end{align}


\subsection{Proof of Proposition 1} 

\begin{proof}
The correlation between $\mathbf{x}_0$ and $\mathbf{x}_T$ is given by~\cite{ho2020denoising}

\begin{align}
\label{eq: dependency}
    \mathbf{x}_T = \sqrt{\overline{\alpha}_T}\mathbf{x}_0 + \sqrt{1-\overline{\alpha}_T}\mathbf{\epsilon}, \mathbf{\epsilon}\sim\mathcal{N}(0, \mathbf{I})\FullStop
\end{align}
It follows from Equation~\ref{eq: dependency} that
\begin{align*}
	\frac{\sqrt{\overline{\alpha_T}}\tilde{\mathbf{x}}_0-\mathbf{x}_T}{\sqrt{1-\overline{\alpha_T}}}
	&=\frac{\sqrt{\overline{\alpha_T}}\tilde{\mathbf{x}}_0-\left( \sqrt{\overline{\alpha}_T}\mathbf{x}_0 + \sqrt{1-\overline{\alpha}_T}\mathbf{\epsilon} \right)}{\sqrt{1-\overline{\alpha_T}}}	
    \\ 
    &= \beta (\tilde{\mathbf{x}}_0 - \mathbf{x}_0) - \epsilon\Comma
\end{align*}
with $\beta = \sqrt{\frac{\overline{\alpha}_T}{1-\overline{\alpha}_T}}$. Thus, under the condition $\|\frac{\sqrt{\overline{\alpha_T}}\tilde{\mathbf{x}}_0-\mathbf{x}_T}{\sqrt{1-\overline{\alpha_T}}}\|^2_2 \geq M$, we have
\begin{align*}
	\mathbb{E}[\mathcal{L}_{\text{contrastive}}]
	&=\mathbb{E} \left[ \left\Vert \frac{\sqrt{\overline{\alpha_T}}\tilde{\mathbf{x}}_0-\mathbf{x}_T}{\sqrt{1-\overline{\alpha_T}}} \right\Vert_2^2 - M \right]\\
	&=\mathbb{E} \left[ \Vert \beta (\tilde{\mathbf{x}}_0 - \mathbf{x}_0) - \epsilon \Vert_2^2 - \Vert \epsilon \Vert_2^2 \right]\\
	&=\mathbb{E} \left[ \Vert \beta (\tilde{\mathbf{x}}_0  - \mathbf{x}_0)\Vert_2^2 - 2 \left< \beta (\tilde{\mathbf{x}}_0 - \mathbf{x}_0), \epsilon \right> \right]\\
	&=\mathbb{E} \left[ \Vert \beta (\tilde{\mathbf{x}}_0  - \mathbf{x}_0)\Vert_2^2 \right] - 2 \beta \mathbb{E} \left[ \left< (\tilde{\mathbf{x}}_0 - \mathbf{x}_0), \epsilon \right> \right]\FullStop
\end{align*}
However, since $\tilde{\mathbf{x}}_0 - \mathbf{x}_0$ and $\epsilon$ are independent and $\mathbb{E}[\epsilon]=0$, we have
$$\mathbb{E} \left[ \left< (\tilde{\mathbf{x}}_0 - \mathbf{x}_0), \epsilon \right> \right]
	= \left< \mathbb{E} \left[ \tilde{\mathbf{x}}_0 - \mathbf{x}_0 \right], \mathbb{E} \left[ \epsilon \right] \right>
	= 0\FullStop
	$$
Thus, 
\begin{align*}
\mathbb{E}[\mathcal{L}_{\text{contrastive}}] = \mathbb{E} \left[ \Vert \beta (\tilde{\mathbf{x}}_0  - \mathbf{x}_0)\Vert_2^2 \right] = \beta^2 \mathbb{E} \left[ \Vert \tilde{\mathbf{x}}_0  - \mathbf{x}_0 \Vert_2^2 \right]\FullStop
\end{align*}
Hence, with $C= \beta^{-2}$, we have $\mathbb{E} \left[ \Vert \tilde{\mathbf{x}}_0  - \mathbf{x}_0 \Vert_2^2 \right] \leq C \delta$ as desired.
\end{proof}

\section{Remark on Related Works}\label{sec: related-work}

\textbf{Grasp Datasets.} Numerous grasp datasets have been introduced recently~\cite{newbury2023deep}, each with varying characteristics such as data representation (RGB-D or 3D point clouds), grasp labels (rectangle-based or 6-DoF), and quantity~\cite{pourpanah2022review}. Our Grasp-Anything++ dataset differs primarily in its universality, contrasting the limited object selection in existing benchmarks. It covers a wide range of everyday objects and includes natural scene descriptions, facilitating research in \textit{language-driven grasp detection}. Furthermore, the Grasp-Anything++ dataset uniquely presents natural object arrangements, in contrast to the more strictly controlled configurations in previous datasets~\cite{platt2023grasp}. Grasp-Anything++ outperforms other benchmarks in both the number of objects and the number of samples.

\textbf{Contrastive Loss for Diffusion Models.} Recent advancements in contrastive learning have become a prominent attraction in diffusion model research, as evidenced in~\cite{han2023contrastive}. While most studies in the diffusion literature regard contrastive learning primarily as a method of data augmentation~\cite{hua2023multimodal} for improving the performance of models on fine-grained prediction~\cite{chen2023dual}, several notable works, including~\cite{wu2023diffusion, han2023contrastive, lee2023codi, yang2023zero}. For instance, Yang \textit{et al.}~\cite{yang2023zero}  leverage intermediate layer features for calculating contrastive loss in negative sample pairs. Contrary to these perspectives, our paper considers contrastive learning as an integral aspect of the training objective, explicitly contributing to the denoising process of diffusion models.

\textbf{Grasp Detection.} Deep learning has significantly advanced grasp detection, with initial efforts by Lenz \textit{et al.}\cite{lenz2015deep} employing deep learning for grasp pose detection. Following this, deep learning-based approaches\cite{yan2018learning, liang2019pointnetgpd, jiang2021synergies, wen2022catgrasp, ainetter2021end, kumra2020antipodal, cao2023nbmod} have become the predominant methodology in the field. Despite extensive research, the real-world application of deep learning for robotic grasping remains a challenge, primarily due to the limited size and diversity of grasp datasets~\cite{gilles2022metagraspnet, platt2023grasp}.

\begin{figure}[h]
    \centering
    \includegraphics[height=0.8\linewidth]{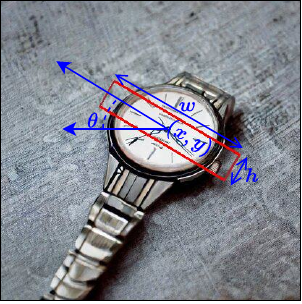}
    \vspace{-2ex}
    \caption{\textbf{Rectangle representation of grasp poses.} The groundtruth rectangle is defined with $5$ parameters $\{x, y, w, h, \theta \}$ which $(x, y)$ is the center point of the rectangle, $(w, h)$ is the width and height of the rectangle, and $\theta$ is the rotational angle of the rectangle with respect to the image plane.}
    \label{fig: rectangle-grasp poses}
\end{figure}

\textbf{Grasp Ground Truth Definition.} 6-DoF and rectangle representations are the most prevalent in grasp detection literature. While 6-DoF poses offer greater flexibility and adaptability~\cite{fang2020graspnet} for complex tasks, rectangle grasp poses are advantageous for their simplicity~\cite{jiang2011efficient}, efficiency in specific scenarios, and lower hardware and computational requirements~\cite{depierre2018jacquard}. Considering that objects in synthesized images from foundation models often lack 3D information~\cite{kamon1996learning}, the rectangle representation of grasp poses appears more suitable for our task. The ground truth rectangle is defined in Fig.~\ref{fig: rectangle-grasp poses}. 


\section{Grasp-Anything++ Analysis}\label{sec: dataset}

\textbf{Additional Visualization.} Fig.~\ref{fig:additional-visualization} provides further examples from the Grasp-Anything++ dataset, illustrating its diverse and extensive representation of everyday objects. The additional samples of Grasp-Anything++ showcase a diverse collection of objects typically found in everyday environments, such as home offices, kitchens, and living spaces. The collection includes a variety of shapes, sizes, and types of objects such as writing instruments, electronic devices, and household items, each situated within its own designated space. Furthermore, the annotated grasp poses, produced by the Grasp-Anything++'s pipeline, demonstrate high fidelity, thereby offering a foundation for both qualitative and quantitative grasp detection research.




\begin{figure*}[]
  \centering
  \setlength{\tabcolsep}{2pt}
  \resizebox{\textwidth}{!}{%
    \begin{tabular}{ccccc}
      \includegraphics[width=0.19\textwidth]{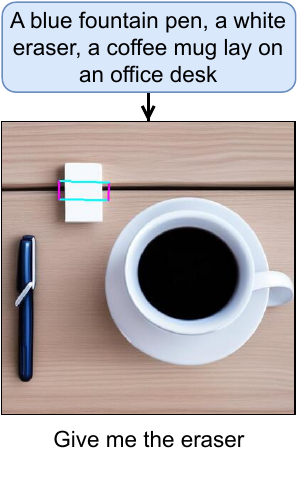} &
      \includegraphics[width=0.19\textwidth]{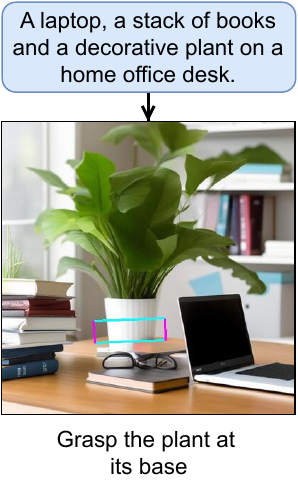} &
      \includegraphics[width=0.19\textwidth]{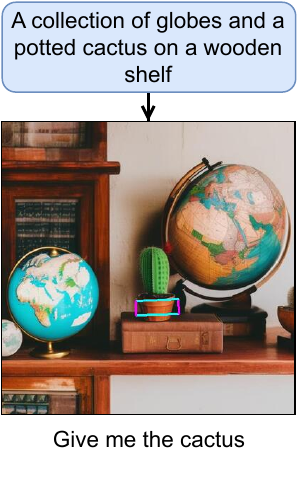} &
      \includegraphics[width=0.19\textwidth]{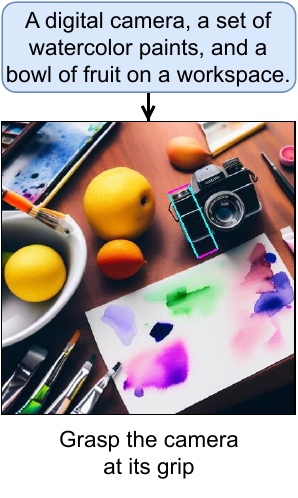} &
      \includegraphics[width=0.19\textwidth]{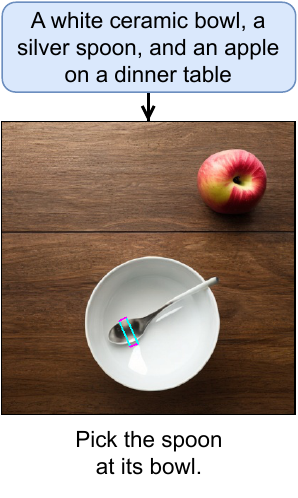} \\ [-4.0mm]
      \includegraphics[width=0.19\textwidth]{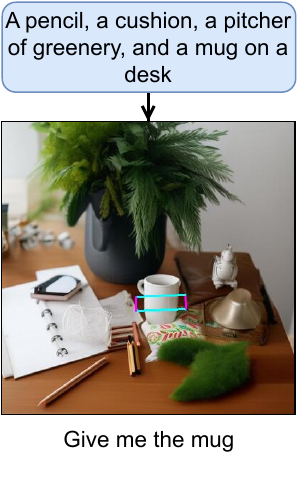} &
      \includegraphics[width=0.19\textwidth]{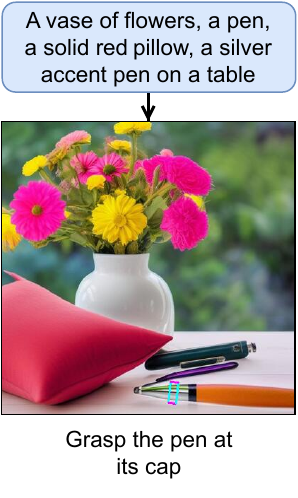} &
      \includegraphics[width=0.19\textwidth]{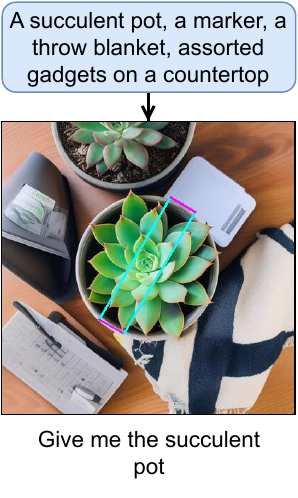} &
      \includegraphics[width=0.19\textwidth]{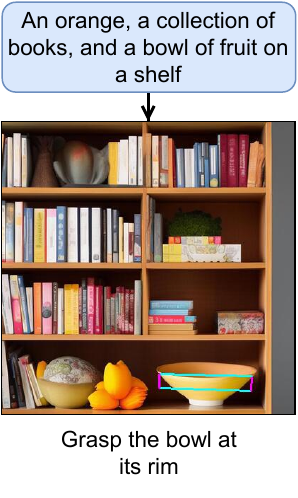} &
      \includegraphics[width=0.19\textwidth]{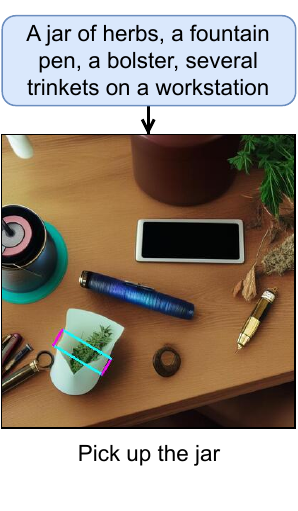} \\ [-4.0mm] 
      \includegraphics[width=0.19\textwidth,]{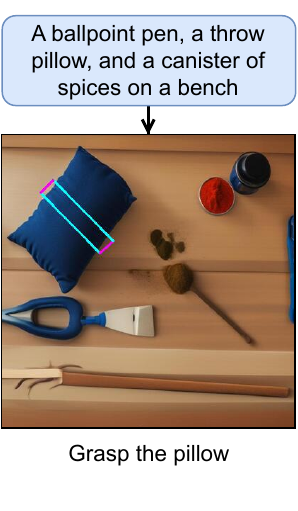} &
      \includegraphics[width=0.19\textwidth]{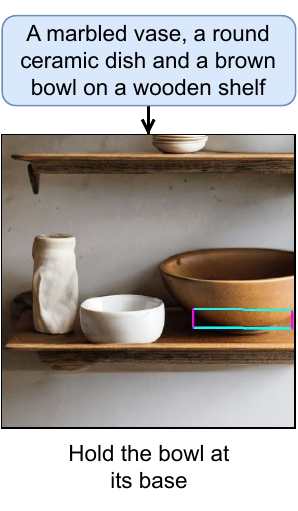} &
      \includegraphics[width=0.19\textwidth]{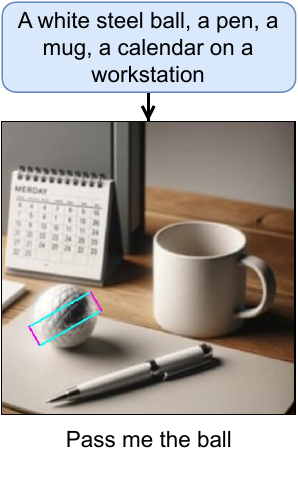} &
      \includegraphics[width=0.19\textwidth]{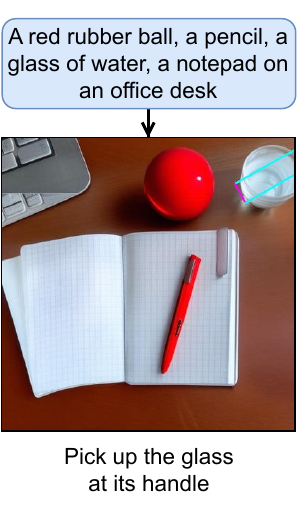} &
      \includegraphics[width=0.19\textwidth]{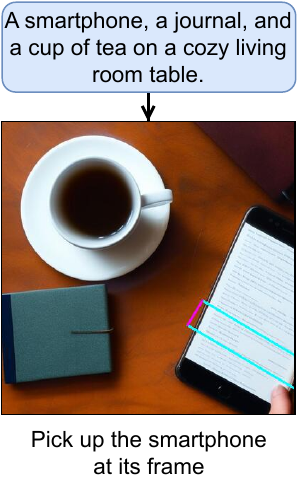} \\ [-3.0mm]
      \includegraphics[width=0.19\textwidth,]{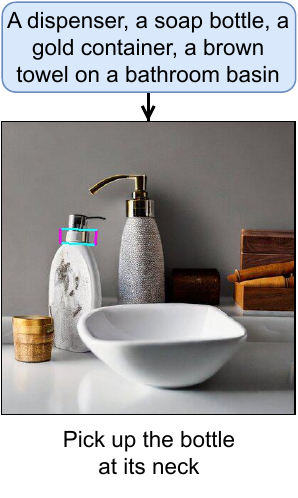} &
      \includegraphics[width=0.19\textwidth]{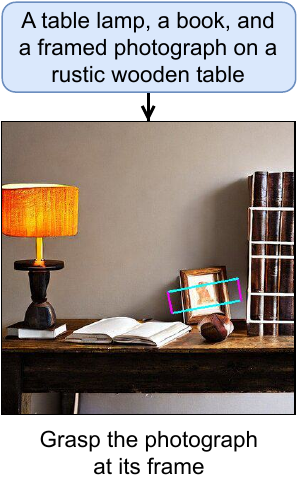} &
      \includegraphics[width=0.19\textwidth]{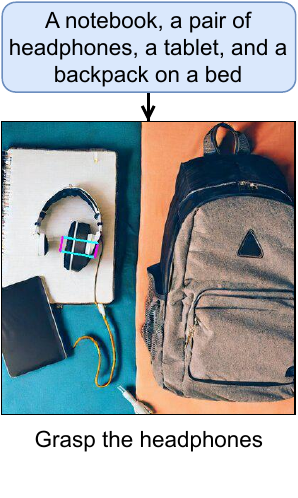} &
      \includegraphics[width=0.19\textwidth]{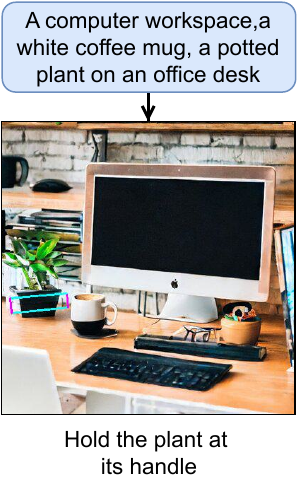} &
      \includegraphics[width=0.19\textwidth]{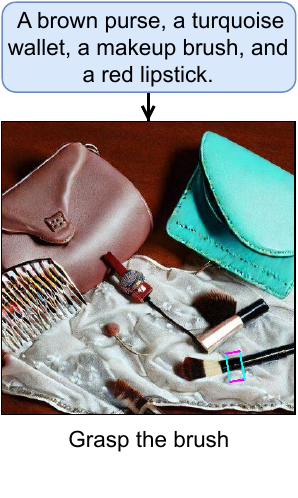} \\ [-4mm]
    \end{tabular}
  }
  \caption{\textbf{
  More samples of our Grasp-Anything++ dataset.}}
  \label{fig:additional-visualization}
\end{figure*}


\section{LGD Implementation Details}\label{sec: implementation-detail}

In this section, we first discuss the observation about the attention mask during the diffusion process which led to our motivation for using the contrastive diffusion. We then provide the implementation details of our LGD network.

\subsection{Observation} In Fig.~\ref{fig: observation}, we present a visualization of the attention mask generated by the vision transformer backbone and the grasp pose throughout the diffusion process. It is evident that during the initial time steps, there is a substantial overlap between the guiding region and the grasp pose, which diminishes as time progresses. This finding suggests that, by the end of the forward process, the guiding region and the noisy grasp pose can be regarded as a contrasting pair. We mathematically express this contrastive relationship between the guiding region $\tilde{\mathbf{x}}_0$ and $\mathbf{x}_t$. This contrastive relationship forms a central role in our network design, as depicted in the overview of our method in the main paper.

\begin{figure}[!ht]
    \centering
    \includegraphics[width=1\linewidth]{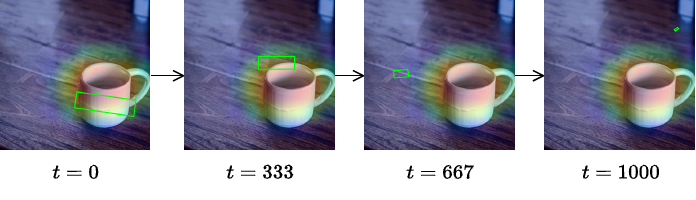}
    \caption{\textbf{Observation.} Comparison between the grasp poses and attention map output by the network backbone.}
    \label{fig: observation}
\end{figure}

\subsection{LGD Implementation Details} For an image $\textit{\textbf{I}}$ of resolution $W\times H$, we employ a ResNet-50 vision encoder backbone~\cite{he2016deep} to derive feature representations $\textit{\textbf{I}}^\prime\in\mathbb{R}^{w\times h}$ with latent dimensions $w$ and $h$. Similarly, we obtain text embedding $e^\prime\in\mathbb{R}^{|D|}$ using a text encoder, such as CLIP~\cite{radford2021learning} or BERT~\cite{devlin2018bert}. The dimensionality of these latent features depends on the text encoder's architecture.

\begin{table*}[!ht]
    \centering
    \begin{tabular}{@{}clccc@{}}
\toprule
Component & Description & Input size & Output size \cr \midrule
(\textit{i}) & A vision encoder (ResNet-50~\cite{he2016deep}) & $[H, W, 3]$ & $[h, w, 3]$ \\
(\textit{ii}-a) & A text encoder (CLIP~\cite{radford2021learning} or BERT~\cite{devlin2018bert}) & Any & $[|D|]$ \\
(\textit{ii}-b) & MLP layers & $[|D|]$ & $[N+1, d_\text{text}]$ \\
(\textit{iii}) & ALBEF~\cite{li2021align} & $[h, w, 3], [N+1, d_\text{text}]$ & $[W, H], [\ell]$ \\
(\textit{iv}) & MLP layers of Equation~\ref{eq: timestep-mlp} & $[1]$ & $[d_{\text{ts}}]$\\
(\textit{v}) & MLP layers of Equation~\ref{eq: vl-mlp} & $[\ell]$ & $[d_{\text{vl}}]$  \\
(\textit{vi}) & MLP layers of Equation~\ref{eq: df-mlp} & $[M]$ & $[d_{\text{df}}]$  \\
(\textit{vii}) & MLP layers to output denoising state & $[d_{\text{ts}}+d_{\text{vl}}+d_{\text{df}}]$ & $[M]$ \\
\bottomrule
\end{tabular}
\captionof{table}{\textbf{Architecture specifications of our method.}}
\label{tab: architecture_lgd}
\end{table*}

\begin{table}[!ht]
    \centering
    \begin{tabular}{@{}lccc@{}}
\toprule
Hyperparameter &  Value \cr \midrule
$W$ & 224 \\
$H$ & 224 \\
$N$  & 196
\\
$M$ (number of grasp parameters) & 5
\\
$|D_{\text{CLIP}}|$ of (\textit{ii}-a) & 512 \\
$|D_{\text{BERT}}|$ of (\textit{ii}-a) & 768 \\
$d_{\text{text}}$ of (\textit{ii}-b) & 128 \\
$\ell$ of (\textit{iii}) & 1024 \\
$d_{\text{ts}}$ of (\textit{iv}) & 32 \\
$d_{\text{vl}}$ of (\textit{v}) & 256 \\
$d_\text{df}$ of (\textit{vi}) & 256 \\
Num. attention layers (ALBEF) & 6 \\
\bottomrule
\end{tabular}
\captionof{table}{\textbf{Hyperparameter details.}}
\label{tab: hyper_params}
\end{table}

Utilizing the ALBEF architecture~\cite{li2021align}, we integrate vision and language embeddings. Specifically, we encode each intermediate feature $\textit{\textbf{I}}^\prime$ into a set $\textit{\textbf{v}} = \{\textit{\textbf{v}}_\text{cls}, \textit{\textbf{v}}_{1}, \textit{\textbf{v}}_{2},\dots, \textit{\textbf{v}}_{N}\}$, where $N$ denotes the number of segmented patches, similar to the approach in~\cite{dosovitskiy2020image}. We feed text embeddings $e^\prime$ through MLP layers, producing a sequence of embeddings $\textit{\textbf{u}}=\{\textit{\textbf{u}}_\text{cls}, \textit{\textbf{u}}_{1}, \ldots, \textit{\textbf{u}}_{N}\}$. Cross-attention mechanisms in the multimodal encoder integrate image features $\textit{\textbf{v}}$ with text features $\textit{\textbf{u}}$. The multimodal attention layer outputs an attention map $\tilde{\mathbf{x}}_0\in\mathbb{R}^{W\times H}$ and a final text-image representation $z_{\text{vl}}^\ast\in\mathbb{R}^{\ell}$.

Using a sequence of MLP layers, we integrate features from timestep $t+1$, the current state $\mathbf{x}_{t+1}$, and the final text-image representation $z\in\mathbb{R}^{\ell}$ as follows
\begin{equation}
\label{eq: timestep-mlp}
    z_{\text{ts}} = \text{MLP}(t+1) \in\mathbb{R}^{d_{\text{ts}}}\FullStop
\end{equation}
\begin{equation}
\label{eq: vl-mlp}
    z_{\text{vl}} = \text{MLP}(z_{\text{vl}}^\ast)\in\mathbb{R}^{d_{\text{vl}}}\FullStop
\end{equation}
\begin{equation}
\label{eq: df-mlp}
    z_{\text{df}} = \text{MLP}(\mathbf{x}_{t+1})\in\mathbb{R}^{d_{\text{df}}}\FullStop
\end{equation}

Finally, we concatenate $z_{\text{ts}}, z_{\text{vl}}, z_{\text{df}}$ to form $z$, which is then processed through an additional MLP to yield the decoded state for the denoising process.
\begin{equation}
    \label{eq: mlp-denoising}
    \mathbf{x}_{t} = \text{MLP}(z)\in\mathbb{R}^{d_{\text{df}}}\FullStop
\end{equation}

\textbf{Architecture Summarization.} As outlined in the main paper, our network architecture contains: (\textit{i}) a vision encoder, (\textit{ii}) a text encoder followed by MLP layers, (\textit{iii}) ALBEF module, (\textit{iv}) MLP layers to encode timestep information, (\textit{v}) MLP layers to encode text-image features $z_{\text{vl}}^\ast$,   (\textit{vi}) MLP layers to encode noisy state $\mathbf{x}_t$, and (\textit{vii}) MLP layers to output denoising state. We summarize the architecture and hyperparameters of LGD in Table~\ref{tab: architecture_lgd} and Table~\ref{tab: hyper_params}.

\section{Experimental Setups}\label{sec-exp-setup} 
We present implementation details of other baselines in the language-driven grasp detection task.

\subsection{Baseline Setups}
\textbf{Linguistic versions of GR-ConvNet~\cite{kumra2020antipodal}, Det-Seg-Refine~\cite{ainetter2021end}, GG-CNN~\cite{morrison2018closing}}. We make slight modifications to these baselines by adding a component to fuse image and text features from the input. Specifically, we utilize the CLIP text encoder~\cite{radford2021learning} to extract text embeddings $e^\prime$. To ensure a fair comparison between methods, we also utilize ALBEF architecture~\cite{li2021align} to do the fusion between the text embedding and the visual features. The remaining training loss and parameter are inherited from the original work.   



\textbf{CLIPORT}~\cite{shridhar2022cliport}. The original CLIPORT architecture learns a policy $\pi$, which does not directly solve our task. We modify the CLIPORT architecture's final layers with appropriately sized MLPs to output grasp poses defined by five parameters $(x, y, w, h, \theta)$. This adaptation ensures consistency with our grasp detection baselines, diverging from CLIPORT's original policy $\pi$ learning framework.

\textbf{CLIP-Fusion}~\cite{xu2023joint}. In our re-implementation of the architecture from~\cite{xu2023joint}, we follow the cross-attention module in CLIP-Fusion with constructed MLP layers. The final MLP layers in the architecture is modified to output five parameters, corresponding to predicted grasp poses.

\begin{figure}[h]
\centering
\def\svgwidth{1\columnwidth}
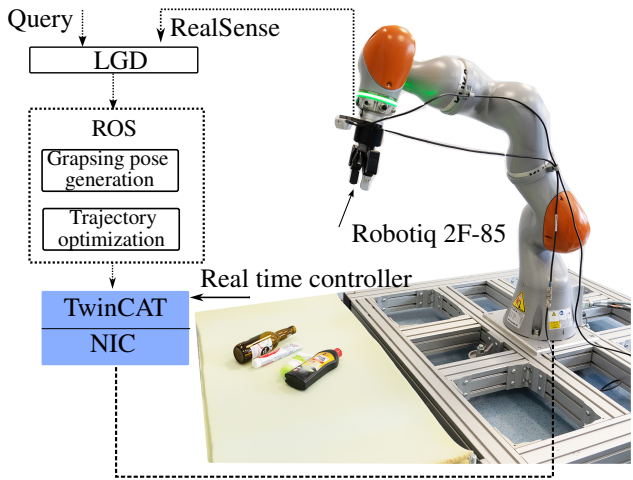

\vspace{1pt}
\caption{\textbf{Overview of the robotic experiment setup.}}
\label{fig: robot demonstration}
\vspace{-0.2ex}
\end{figure}

\begin{figure*}[h]
    \includegraphics[width=\linewidth]{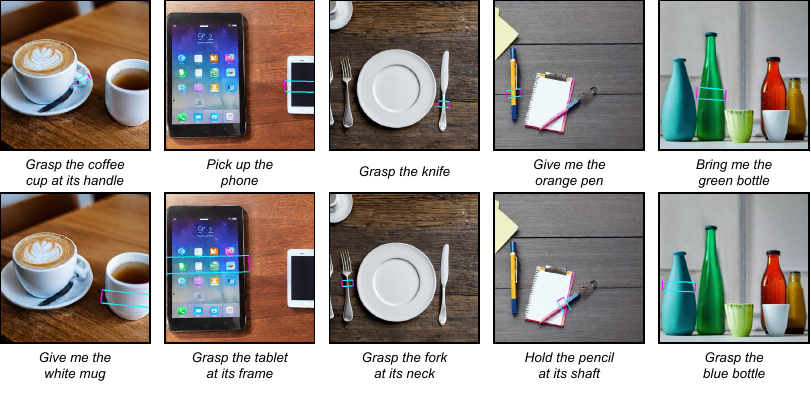}
    \vspace{-3ex}
    \caption{\label{fig: more-visualization} \textbf{Additional language-driven grasp detection visualizations.} We provide fine-grained grasp detection cases, our method successfully differentiates between objects of similar structure but varying color, such as a green bottle and a blue bottle.}
\end{figure*}

\subsection{Robotic Setup}
In Figure~\ref{fig: robot demonstration}, we present the robotic evaluation conducted on a KUKA robot. Our grasp detection leverages our proposed LGD and other methods listed in the real robot experiments (Table 4 of the main paper), and the results are translated into a 6DOF grasp pose through depth images captured by an Intel RealSense D435i depth camera, as in~\cite{kumra2020antipodal}. 
The trajectory planner~\cite{vu2023machine,beck2022singularity} is employed for the execution of the grasp. We use two computers for the experiment. The first computer (PC1) runs the real-time control software Beckhoff TwinCAT, the Intel RealSense D435i camera, and the Robotiq 2F-85 gripper, while the second computer (PC2) runs ROS on Ubuntu Noetic 20.04. PC1 communicates with the robot via a network interface card (NIC) using the EtherCAT protocol. The inference process is performed on PC2 with an NVIDIA 3080 GPU. Our assessment encompasses both single-object and cluttered scenarios, involving a diverse set of $20$ real-world daily objects (Fig.~\ref{fig:20objects}). To ensure robustness and reliability, we repeat each experiment for all methods a total of $30$ times.

\section{Extra Experiments}
\label{sec-extra-results}

\textbf{Number of Parameter Comparison.} Table~\ref{table: computation-time} shows the number of parameters in all methods. This table illustrates that the results from the language-driven grasp detection studies in the main paper reveal a consistent trade-off between performance and the number of parameters across all baselines. Notably, LGD emerges as the balanced baseline, offering a good balance of performance efficiency and computational resource utilization.


\begin{table}[!ht]
    \centering
    \renewcommand
\tabcolsep{4pt}
\hspace{1ex}
\resizebox{\linewidth}{!}{
    \begin{tabular}{@{}rcc@{}}
\toprule
Baseline & $\#$Parameters & Success rate\\ 
\midrule
GR-ConvNet~\cite{kumra2020antipodal} + CLIP~\cite{radford2021learning}  & 2.07M  & 0.24\\ 
Det-Seg-Refine~\cite{ainetter2021end} + CLIP~\cite{radford2021learning} & 1.82M  & 0.20\\ 
GG-CNN~\cite{morrison2018closing} + CLIP~\cite{radford2021learning} & 1.24M  & 0.10\\
CLIPORT~\cite{shridhar2022cliport} & 10.65M & 0.29\\
CLIP-Fusion~\cite{xu2023joint} & 13.51M & 0.33\\
\textbf{LGD (ours)} &  5.18M & 0.45\\
\bottomrule
\end{tabular}
}
\caption{\label{table: computation-time} \textbf{Number of parameter comparison.}}
\end{table}



\begin{figure}[h]
    \centering
    \includegraphics[width=\linewidth]{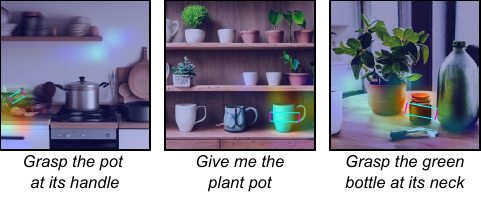}
    \caption{\textbf{Prediction failure cases.} In cases where objects have similar structures, such as jars and bottles, our LGD occasionally fails to detect correct grasp poses.}
    \label{fig:failure-cases}
\end{figure}
\textbf{Failure Cases.} Though achieving satisfactory results, our method still predicts incorrect grasp poses. A large number of objects and grasping prompts in our dataset suggest a significant challenge for the tasks. Some failure cases are depicted in Fig.~\ref{fig:failure-cases}. From this figure, we can see that the correlation between the text and the attention map of the visual features is not well-aligned, which leads to incorrect prediction of the grasp poses.



\begin{figure}[ht]
\centering
\subfloat[]{\label{fig:realsens1}\includegraphics[height=35mm]{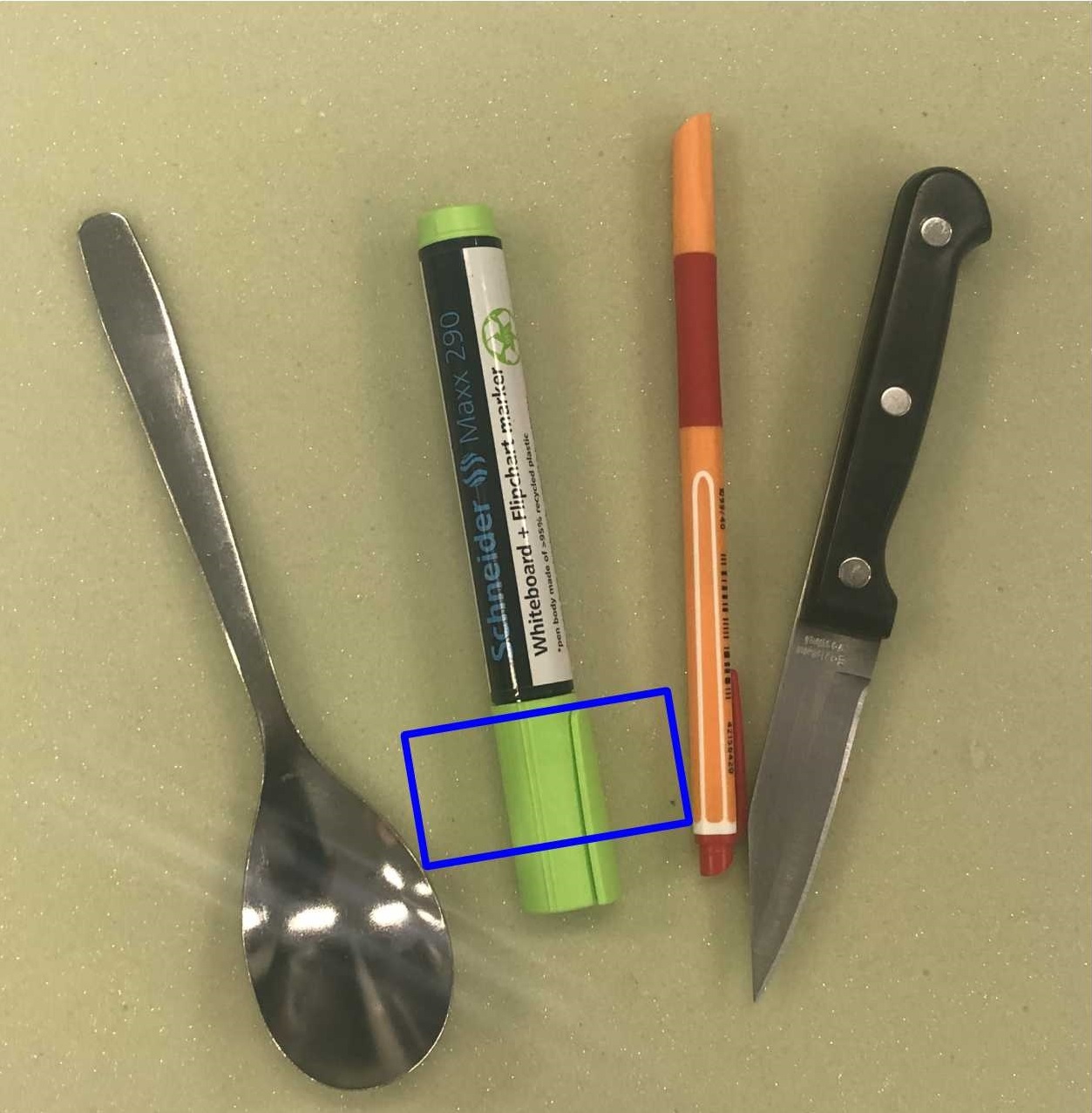}}
\hspace{2ex}
\subfloat[]{\label{fig:realsens2}\includegraphics[height=35mm]{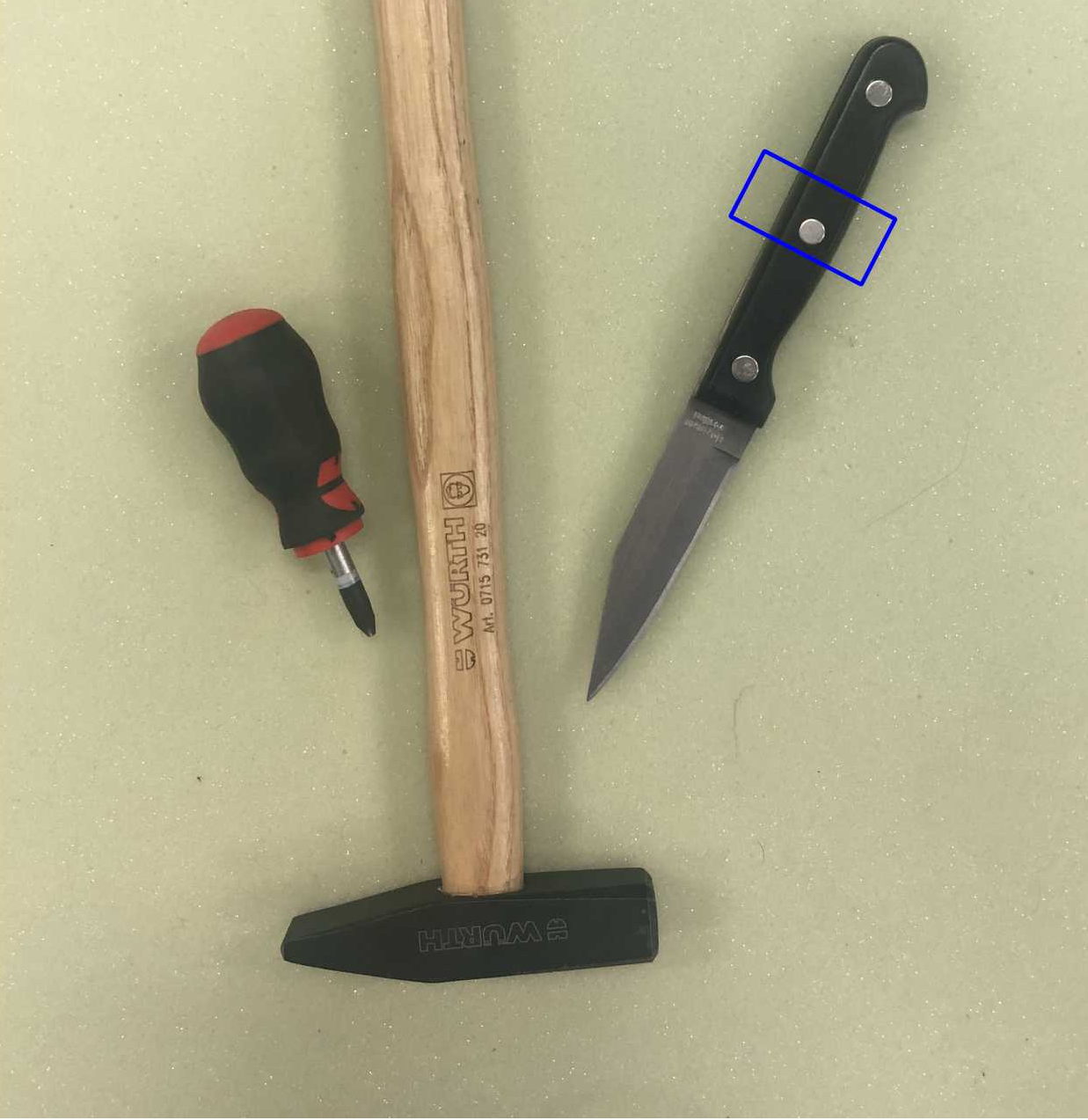}}
\vspace{-1ex}
\caption{\textbf{Detection results in robotic experiments}. Images are captured from a RealSense camera with experiments in Fig~\ref{fig:robot-motion}.}
\label{fig:realsense-result}
\end{figure}

\begin{figure*}[h]
    \centering
    \includegraphics[width=\linewidth, height=0.25\linewidth]{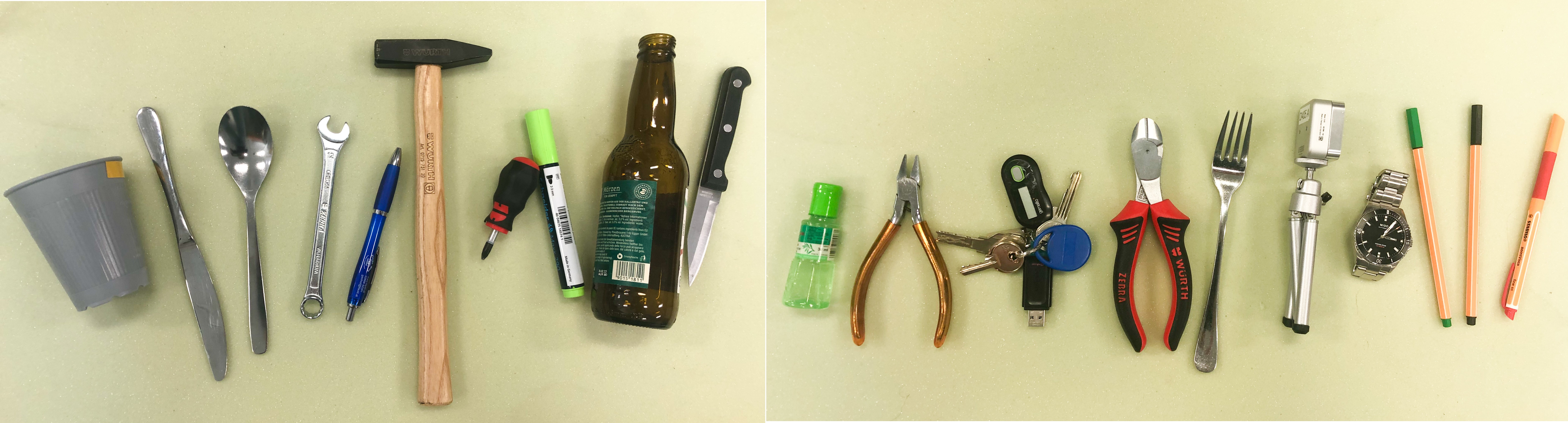}
    \caption{\textbf{Set of 20 objects used in the robotic experiment.}}
    \label{fig:20objects}
\end{figure*}

\textbf{Additional Detection Visualization.} Fig.~\ref{fig: more-visualization} illustrates additional language-driven grasp pose detection using the LGD method. The result demonstrates our method's capability in reasonably aligning grasp poses with linguistic instructions in fine-grained scenarios, as seen with a green and blue bottle. Remarkably, these qualitative examples demonstrate the effectiveness of our proposed LGD method in fine-grained cases, in line with the contrastive loss objectives outlined in our main paper.

\textbf{Robotic Demonstration.} In Fig.~\ref{fig:robot-motion}, we show a sequence of actions when the KUKA robot grasps different objects in cluttered scenes. Fig.~\ref{fig:realsense-result} further shows the detection result of our LGD method on an image captured by a RealSense camera mounted on the robot. The robotic experiments demonstrate that although our LGD method is trained on a synthesis Grasp-Anything++ dataset, it still be able to generalize to detect grasp pose in real-world images. More illustrations can be found in our Demonstration Video.

\begin{figure*}[h]
    \centering
    \includegraphics[width=\linewidth,height=0.35\linewidth]{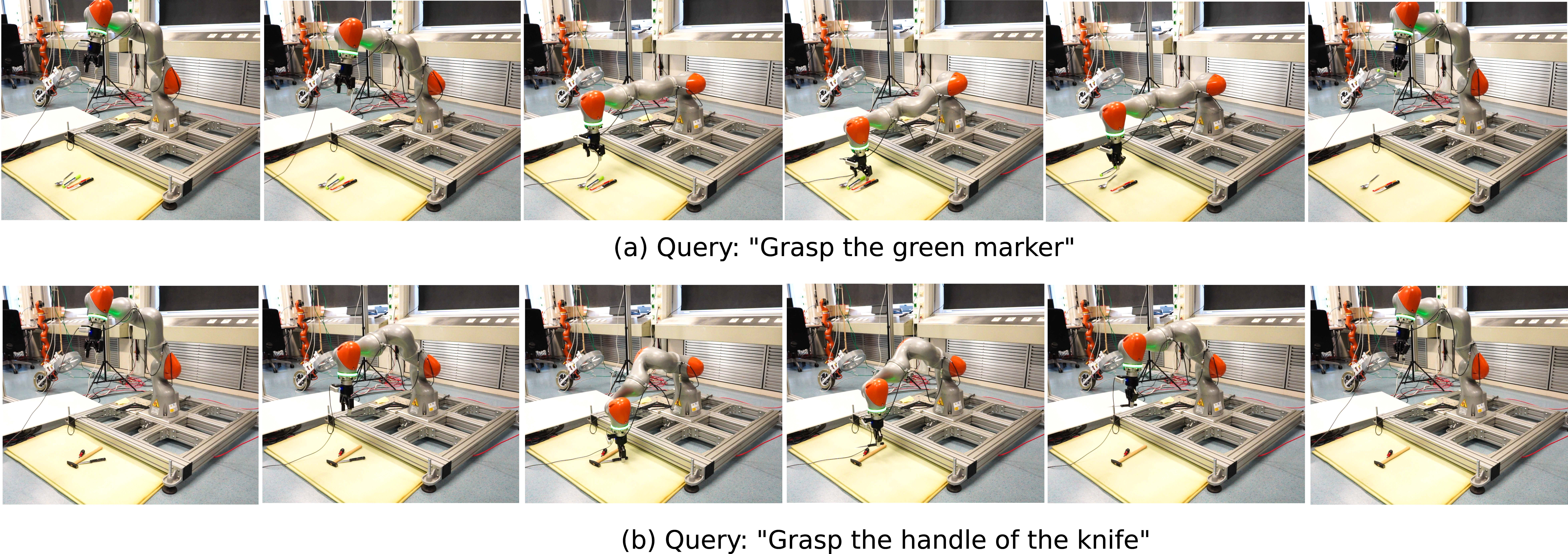}
    \vspace{-2ex}\caption{\textbf{Snapshots of two example robotic experiments.}}
    \label{fig:robot-motion}
\end{figure*}


\end{document}